\documentclass[11pt]{article}
\usepackage{epsfig,amsmath,amssymb,latexsym,color}
\usepackage[hyperfootnotes=false]{hyperref}
\usepackage{multicol}
\usepackage{multirow}
\usepackage{url}
\usepackage{algorithm}
\usepackage{algorithmic}
\usepackage{booktabs} 
\usepackage{wrapfig}
\usepackage{natbib}
\usepackage{tikz-cd}
\usepackage{lineno}

\setcitestyle{citesep={;},square}

\newcounter{mnote}
  \let\oldmarginpar\marginpar
 \renewcommand\marginpar[1]{\-\oldmarginpar[\raggedleft\footnotesize #1]%
    {\raggedright\footnotesize #1}}

\newtheorem{theorem}{Theorem}
\newtheorem{lemma}{Lemma}
\newtheorem{proposition}{Proposition}

\newtheorem{definition}{Definition}
\newtheorem{assumption}{Assumption}
\newtheorem{remark}{Remark}

\newenvironment{proof}{\begin{trivlist}\item[]{\emph{Proof.}}}
               {\hfill$\Box$\end{trivlist}}

\DeclareMathOperator*{\argmin}{arg\,min}

\date{\today}

\title{Transformers for Learning on Noisy and Task-Level Manifolds: Approximation and Generalization Insights} 
\author{
Zhaiming Shen\thanks{\{zshen49, ahavrilla3, wliao60\}@gatech.edu.
School of Mathematics,
Georgia Institute of Technology, Atlanta, GA 30332}
\and
Alex Havrilla\footnotemark[1] 
\and 
Rongjie Lai\thanks{lairj@purdue.edu. Department of Mathematics, Purdue University, West Lafayette, IN 47907} 
\and
Alexander Cloninger\thanks{acloninger@ucsd.edu. Department of Mathematics and Halicio{\u g}lu Data Science Institute, University of California, San Diego, La Jolla, CA 92093}
\vspace{1mm}
\and
Wenjing Liao\footnotemark[1]
}

\begin{document}

\maketitle

\begin{abstract}
    Transformers serve as the foundational architecture for large language  and video generation models, such as GPT, BERT, SORA and their successors. Empirical studies have demonstrated that real-world data and learning tasks exhibit low-dimensional structures, along with some noise or measurement error. The performance of transformers tends to depend on the intrinsic dimension of the data/tasks, though theoretical understandings remain largely unexplored for transformers.   This work establishes a theoretical foundation by analyzing the performance of transformers for  regression tasks involving noisy input data near a manifold. Specifically, the input data are in a tubular neighborhood of a manifold, while the ground truth function depends on the projection of the noisy data onto this manifold, referred to as the task-level manifold. We prove approximation and generalization errors which crucially depend on the intrinsic dimension of the task-level manifold. Our results demonstrate that transformers can leverage  low-complexity structures in  learning task even when the input data are perturbed by high-dimensional noise. Our novel proof technique constructs representations of basic arithmetic operations by transformers, which may hold independent interest. 
\end{abstract}

\section{Introduction}

Transformer architecture, introduced in \citet{Vaswani17}, has reshaped the landscape of machine learning, enabling unprecedented advancements in natural language processing (NLP), computer vision, and beyond.
In transformers,  traditional recurrent and convolutional architectures are replaced by an attention mechanism. Transformers have achieved remarkable success in large language models (LLMs) and video generation, such as GPT \citep{achiam2023gpt}, BERT \citep{devlin2018bert}, SORA \citep{videoworldsimulators2024} and their successors.

Despite the success of transformers, their approximation and generalization capabilities remain less explored compared to other network architectures, such as feedforward and convolutional neural networks. 
Some theoretical investigations of transformers can be found in \citet{jelassi2022vision,Yun19,edelman2022inductive,StatisticallyMeaningfulApproximation,Takakura23,TransformerClassifier,bai2023transformers,tai2025mathematical}. 
Specifically, \citet{Yun19} proved that transformer models can universally approximate continuous sequence-to-sequence functions on a compact support, while
while the network size grows exponentially with respect to the sequence dimension.   \citet{edelman2022inductive} evaluated the capacity of Transformer networks and derived the sample complexity to learn sparse Boolean functions. 
\citet{Takakura23} studied the approximation and estimation ability of Transformers as sequence-to-sequence functions with anisotropic smoothness on infinite dimensional input. 
\citet{TransformerClassifier} studied binary classification  with transformers when the posterior probability function exhibits  a hierarchical composition model with H\"older smoothness. \citet{jelassi2022vision} analyzed a simplified version of vision transformers and showed that they can learn the spatial structure and generalize.  \citet{lai2024attention}
established a connection between transformers and smooth cubic splines. 
From a computational perspective, the transformer network architecture was connected to integral differential equations in \citet{tai2025mathematical}.
The in-context learning capability of transformers was investigated in \citet{bai2023transformers,zhang2024trained,garg2022can,SHLL25} and many others.

Compared to transformers, feedforward and convolutional neural networks are significantly better understood in terms of approximation \citep{Cybenko89,Hornik89,Leshno93,Mhaskar93,B17,M99,Pinkus99,P98,Y17,lu2021deep,oono2019approximation,LaiShen21,LaiShen24,zhou2020universality} and generalization \citep{kohler2011analysis,schmidt2020nonparametric,oono2019approximation} theories. Theoretical results in \citet{Y17,lu2021deep,oono2019approximation,schmidt2020nonparametric} addressed function approximation and estimation in a Euclidean space. \citet{schonsheck2019chart,schonsheck2022semi,liu2024deep} studied the approximation and generalization errors of deep ReLU networks for representing low-dimensional manifolds. For functions supported on a low-dimensional manifold, approximation and generalization theories were established for feedforward neural networks in \citet{Chui18,shaham2018provable,chen2019efficient,SH19,nakada2020adaptive,Chen22}  and for convolutional residual neural networks in \citet{liu2021besov}.
To relax the exact manifold assumption and allow for noise on input data, \citet{Cloninger21} studied approximation properties of feedforward neural networks under inexact manifold assumption, i.e., data are in a tubular neighborhood of a manifold and the groundtruth function depends on the projection of the noisy data onto the manifold. 
This relaxation accommodates input data with noise and accounts for the low complexity of the learning task beyond the low intrinsic dimension of the input data, making the theory applicable to a wider range of practical scenarios for feedforward neural networks.

In the application of transformers,
empirical studies have demonstrated that image, video, text data and learning tasks tend to exhibit low-dimensional structures \citep{Pope2021TheID,KaplanIDScalingLaws,Havrilla24}, along with some noise or measurement error in real-world data sets. The performance of transformers tends to depend on the intrinsic dimension of the data/tasks \citep{KaplanIDScalingLaws,Havrilla24,Razzhigaev2023TheSO,min2023an,Aghajanyan2020IntrinsicDE}.
Specifically, \citet{Aghajanyan2020IntrinsicDE} empirically showed that common pre-trained models in NLP have a very low intrinsic dimension. 
\citet{Razzhigaev2023TheSO,Havrilla24} investigated the intrinsic dimension of token embeddings in transformer architectures, and  obtained a significantly lower intrinsic dimension than the token dimension. 

Despite of the empirical findings connecting to performance of transformers with the low intrinsic dimension of data/tasks,  theoretical understandings about how transformers adapt to low-dimensional data/task structures and build robust predictions against noise are largely open. \citet{Havrilla24} analyzed the approximation and generalization capability of transformers for regression tasks when the input data exactly lie on a low-dimensional manifold. However, the setup in \citet{Havrilla24} does not account for noisy data concentrated near a low-dimensional manifold and low-complexity in the regression function.

In this paper, we bridge this theoretical gap by analyzing the approximation and generalization error of transformers for regression of functions on a tubular neighborhood of a manifold. To leverage the low-dimensional structures in the learning task, the function depends on the projection of the input onto the manifold.
Specifically, let $\mathcal{M} \subseteq [0,1]^D$ be a compact, connected $d$-dimensional Riemannian  manifold isometrically embedded in $\mathbb{R}^D$ with a positive reach $\tau_{\mathcal{M}}$, and $\mathcal{M}(q)$ be a tubular region around the manifold $\mathcal{M}$ with local tube radius given by $q \in [0,1)$ times the local reach (see Definitions \ref{defmanifold} and \ref{deftubularm}).
We consider function $f:\mathcal{M}(q) \to\mathbb{R}$ in the form:  \begin{equation}f(x)=g(\pi_{\mathcal{M}}(x)), \ \ \forall x \in  \mathcal{M}(q)
\label{eq:f}
\end{equation}
where 
\begin{equation} 
\pi_{\mathcal{M}}(x)=\argmin_{z\in \mathcal{M}}\|x-z\|_2, 
\label{proj}
\end{equation}
is the orthogonal projection onto the manifold $\mathcal{M}$, and $g:\mathcal{M}\to\mathbb{R}$ is an unknown $\alpha$-H{\"o}lder function  on the  manifold $\mathcal{M}$. An illustration of the tubular region and the orthogonal projection onto the manifold is shown in Figure~\ref{Mqfig}. 

\begin{figure}[t]
    \centering
\includegraphics[width=0.5\linewidth]{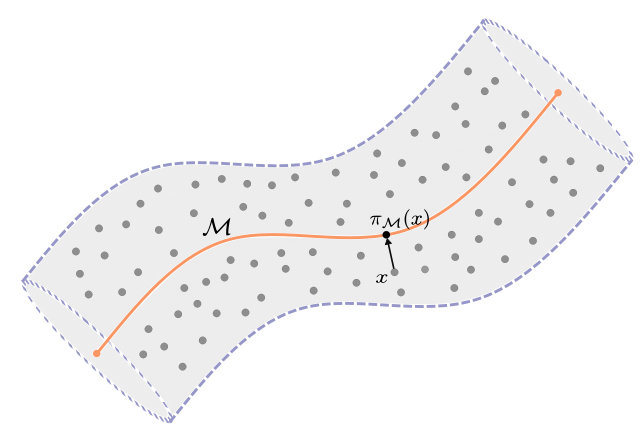}
    \vspace{-1mm}
    \caption{\small The tubular region around manifold $\mathcal{M}$ and the orthogonal projection $\pi_{\mathcal{M}}$.}
    \label{Mqfig}
\end{figure}

The regression model in \eqref{eq:f} covers a variety of interesting scenarios: 1) {\bf Noisy Input Data:} The input $x$ is a perturbation of its clean counterpart $\pi_{\mathcal{M}}(x)$ on the manifold $ \mathcal{M}$. One can access the input and output pairs, i.e. $(x,f(x))$ but the clean counterpart $\pi_{\mathcal{M}}(x)$ is not available in this learning task. 2) {\bf Low Intrinsic Dimension in the Machine Learning Task:} The input data live in a high-dimensional space $\mathbb{R}^D$, but the regression or inference task has a low complexity. In other words, the output $f(x)$ locally depends on $d$ tangential directions on the task manifold $\mathcal{M}$, and the function is locally invariant along   the $D-d$ normal directions on the manifold.
The model in \eqref{eq:f} is also general enough to include many interesting special cases. For example, when $\mathcal{M}$ is a linear subspace, the model in \eqref{eq:f} becomes the well-known multi-index model \citep{cook2002dimension}.  When $q=0$, one recovers the exact manifold regression model where functions are supported exactly on a low-dimensional manifold.   

The regression model in \eqref{eq:f} has been considered in several existing works. 
In \cite{Cloninger21}, the regression function in \eqref{eq:f} is approximated and estimated by a feedforward neural network. 
Despite of deep learning approaches, the regression model in \eqref{eq:f} is studied in \cite{wu2024conditional} using a nonparametric approach with performance guarantees  when $\mathcal{M}$ is a one-dimensional curve. When the task manifold $\mathcal{M}$ is reduced to a linear subspace, the single-index or multi-index models have been extensively studied in literature \citep{li1991sliced,hristache2001structure,li2005contour,constantine2014active,lanteri2022conditional}.

In this paper, we establish   mathematical approximation and statistical estimation (or generalization) theories  for functions in \eqref{eq:f} via transformer neural networks.
\vspace{3mm}
\\
{\bf Approximation Theory:} Under proper assumptions of $\mathcal{M}$, for any $\epsilon >0$, there exists a transformer neural network to universally approximate function $f$ in \eqref{eq:f} up to $\epsilon$ accuracy (Theorem \ref{thmApprox}). The width of this transformer network is in the order of $D\epsilon^{-\frac{d}{\alpha}}$ and the depth is in the order of $d+\ln(\ln(\epsilon^{-1}))$. Note that $d$ is the intrinsic dimension of the manifold $\mathcal{M}$ and $\alpha$ represents the H{\"o}lder smoothness of $g$. In this result, the network complexity crucially depends on the intrinsic dimension.
\vspace{3mm}
\\
\noindent
{\bf Generalization Theory:} When $n$ i.i.d. training samples $\{(x_i,f(x_i))\}_{i=1}^n$ are given, we consider the empirical risk minimizer $\hat{\rm T}$ to be defined in \eqref{hatT}. Theorem \ref{thmGeneralization} shows that the squared generalization error of $\hat{\rm T}$ is upper bounded in the order of $n^{-\frac{2\alpha}{2\alpha+d}}$. In the exact manifold case when $q=0$, Theorem \ref{thmGeneralization} gives rise to the min-max regression error \citep{gyorfi2006distribution}. In the noisy case when $q\in (0,1)$, Theorem \ref{thmGeneralization} demonstrates a denoising phenomenon given by transformers such that when the sample size $n$ increases, the generalization error converges to $0$ at a fast rate depending on the intrinsic dimension $d$.
\vspace{3mm}
\\
{\bf Basic Arithmetic Operations Implemented by Transformers:} In addition, our proof explicitly constructs transformers to implement basic arithmetic operations, such as addition, constant multiplication, product, division, etc. Such implementation can be done efficiently (e.g., in parallel) on different tokens.
 These results can be applied individually as building blocks for approximation studies using Transformers.

\vspace{3mm}
\noindent
This paper is organized as follows. In Section~\ref{secPrelim}, we introduce some preliminary definitions. In Section~\ref{secAGTheory}, we present our main results, including the approximation and generalization error bound achieved by transformer networks.  In Section~\ref{secPfSketch}, we provide a proof sketch of our main results. {In Section \ref{sec:Exp}, we conduct two experiments: One experiment demonstrates the advantage of transformers over feed-forward neural networks and to support our main results, and the other experiment explores the denoising effect of a pre-trained vision transformer.} Finally, Section \ref{sec:conclusion} concludes the paper with discussions.

\section{Preliminaries} \label{secPrelim}

\subsection{Manifold}
\begin{definition}[Manifold]
\label{defmanifold}
An $d$-dimensional \emph{manifold} $\mathcal{M}$ is a topological space where each point has a neighborhood that is homeomorphic to an open subset of $\mathbb{R}^d$. Further, distinct points in $\mathcal{M}$ can be separated by disjoint neighborhoods, and $\mathcal{M}$ has a countable basis for its topology.
\end{definition}

\begin{definition} [Medial Axis]
    Let $\mathcal{M}\subseteq\mathbb{R}^D$ be a connected and compact $d$-dimensional submanifold. Its \emph{medial axis} is defined as 
\begin{align*}
\begin{split}
    \text{Med}
    (\mathcal{M}):=\{x\in\mathbb{R}^D \text{ }| &\text{ }\exists p\neq q\in\mathcal{M},
    \|p-x\|_2 
    =\|q-x\|_2=\inf_{z\in\mathcal{M}}\|z-x\|_2\},
\end{split}
\end{align*}
which contains all points $x\in\mathbb{R}^D$ with set-valued orthogonal projection $\pi_{\mathcal{M}}(x)=\argmin_{z\in \mathcal{M}}\|x-z\|_2$.
\end{definition}

\begin{definition} [Local Reach and Reach of a Manifold]
    The \emph{local reach} for $v\in\mathcal{M}$ is defined as
$  \tau_{\mathcal{M}}(v):=\inf_{z\in \text{Med}(\mathcal{M})}\|v-z\|_2,
$
which describes the minimum distance needed to travel from $v$ to the closure of medial axis. The smallest local reach $\tau_{\mathcal{M}}:=\inf_{v\in\mathcal{M}}\tau_{\mathcal{M}}(v)$ is called \emph{reach} of $\mathcal{M}$.
\end{definition}

\begin{definition} [Tubular Region around a Manifold]
\label{deftubularm}
    Let $q\in [0,1)$. The \emph{tubular region} around the manifold $\mathcal{M}$ with local tube radius $q\tau_{\mathcal{M}}(v)$ is defined as
    \begin{align}
\mathcal{M}(q) := \{x\in\mathbb{R}^D \text{ } | \text{ } x=v+u, v\in\mathcal{M}, u\in ker(P(v)^{\top}),\|u\|_2 < q\tau_{\mathcal{M}}(v) \}, \label{tube}
    \end{align}
where the columns of $P(v)\in\mathbb{R}^{D\times d}$ represent an orthonormal basis of the tangent space of $\mathcal{M}$ at $v$. 
\end{definition}

\begin{definition}[Geodesic Distance]
    The \emph{geodesic distance} between $v,v'\in\mathcal{M}$ is defined as 
    \begin{align*}
        d_{\mathcal{M}}(v,v'):=\inf\{|\gamma|:  \gamma\in C^1([t,t']), \gamma:[t,t']\to\mathcal{M}, \gamma(t)=v, \gamma(t')=v'\},
    \end{align*}
where the length is defined by $|\gamma|:=\int_t^{t'}\|\gamma'(s)\|_2 ds$.
The existence of a length-minimizing geodesic $\gamma:[t,t']\to\mathcal{M}$ between any two points $v=\gamma(t),v'=\gamma(t')$ is guaranteed by Hopf–Rinow theorem \citep{hopf1931}.
\end{definition}

\begin{definition} [$\delta$-Separated and Maximal Separated Set]
    Let $S$ be a set associated with a metric $d$, we say $Z\subseteq S$ is \emph{$\delta$-separated} if for any $z,z'\in Z$, we have $d(z,z')>\delta$. We say $Z\subseteq S$ is \emph{maximal separated $\delta$-net} if adding another point in $Z$ destroys the $\delta$-separated property.
\end{definition}

\begin{definition}[Covering Number] \label{coveringDef}
    Let \( (\mathcal{H}, \rho) \) be a metric space, where \( \mathcal{H} \) is the set of objects and \( \rho \) is a metric. For a given \( \epsilon > 0 \), the \emph{covering number} \( \mathcal{N}(\epsilon, \mathcal{H}, \rho) \) is the smallest number of balls of radius \( \epsilon \) (with respect to \( \rho \)) needed to cover \( \mathcal{H} \). More precisely,
    \begin{align*}
        &\mathcal{N}(\epsilon, \mathcal{H}, \rho) := \min \{ N \in \mathbb{N}\text{ }|\text{ } \exists \{h_1, h_2, \dots, h_N\} \subseteq \mathcal{H}, \\
        &\quad\quad\quad\quad\quad\quad\quad\quad \ \forall h \in \mathcal{H} ,\ \exists h_i \text{ such that } \rho(h, h_i) \leq \epsilon \}.
    \end{align*}

\end{definition}

Let $d_{\mathcal{M}}$ be a geodesic metric defined on $\mathcal{M}$, we can extend $d_{\mathcal{M}}$ to the tubular region $\mathcal{M}(q)$ such that
\[d_{\mathcal{M}(q)}(u,v):=d_{\mathcal{M}}(\pi_{\mathcal{M}}(u),\pi_\mathcal{M}(v)),\] 
provided that $u,v\in\mathcal{M}(q)$ has the unique orthogonal projection onto $\mathcal{M}$. According to  \citet[Lemma 2.1]{Cloninger21}, for any $x\in\mathcal{M}(q)$ with $q\in [0,1)$, $x$ has a unique projection onto $\mathcal{M}$ such that $\pi_{\mathcal{M}}(x)=v$.


\subsection{Transformer Network Class}


\begin{definition} [Feed-forward Network Class]
The feed-forward neural network (FFN) class with weights $\theta$ is
\begin{align*}
    \mathcal{FFN}( L_{\rm FFN},& w_{\rm FFN}) =\{{\rm FFN}(\theta;\cdot)\text{ }|\text{ }{\rm FFN}(\theta;\cdot) \text{ }\text{is a FNN} \text{ with at most }L_{\rm FFN} \text{ }\text{layers and width} \text{ } w_{\rm FFN}
    \}.
\end{align*}
\end{definition}
We use ReLU function $\sigma(x)=\max(x,0)$ as the activation function in the feed-forward network. Note that each feed-forward layer is applied tokenwise to an embedding matrix $H$. 

\begin{definition}[Attention and Multi-head Attention]
The attention with the query, key, value matrices $Q,K,V\in\mathbb{R}^{d_{embed}\times d_{embed}}$ is 
\begin{equation}
   \textstyle A_{K,Q,V}(H)=VH\sigma((KH)^\top QH).
\end{equation}
It is worthwhile to note that the following formulation (when the activation function $sigma$ can be applied pointwise) is convenient when analyzing the interaction between a pair of tokens, which is more relevant to us. 
\begin{equation}
\textstyle    A(h_i)=\sum_{i=1}^{\ell}\sigma(\langle Qh_i, Kh_j\rangle)Vh_j
\end{equation}
The multi-head attention (MHA) with $m$ heads is 
\begin{equation}
  \textstyle  {\rm MHA}(H)=\sum_{j=1}^m V_jH\sigma((K_jH)^\top Q_jH).
\end{equation}
\end{definition}

In this paper, we consider ReLU as the activation function rather than Softmax in the attention.

\begin{definition} [Transformer Block]
The transformer block is a residual composition of the
form
\begin{equation}
       {\rm B}(H) = {\rm FFN}({\rm MHA}(H) + H) + {\rm MHA}(H) + H.
\end{equation}

\end{definition}

\begin{definition}[Transformer Block Class]
The transformer block class with weights $\theta$ is 
\begin{align*}
    \mathcal{B}(m, L_{\rm FFN}, w_{\rm FFN}) = \{{\rm B}(\theta;\cdot)\text{ }|&\text{ } {\rm B}(\theta;\cdot) \text{ }\text{a {\rm MHA} with }  m  \text{ attention heads, and a {\rm FNN} layer } 
    \\
    & \text{with depth } L_{\rm FFN} \text{ }\text{and width } w_{\rm FFN}\}.
\end{align*}
    
\end{definition}

\begin{definition} [Transformer Network]
    A transformer network $T(\theta;\cdot)$ with weights $\theta$ is a composition of an embedding layer, a positional encoding matrix, a sequence transformer blocks, and a decoding layer, i.e.,
    \begin{equation} \label{Trep}
        {\rm T}(\theta;x):={\rm DE}\circ {\rm B}_{L_T}\circ\cdots\circ {\rm B}_1\circ ({\rm PE}+{\rm E}(x)),
    \end{equation}
    where  $x\in\mathbb{R}^D$ is the input, ${\rm E}:\mathbb{R}^D\to\mathbb{R}^{d_{embed}\times\ell}$ is the linear embedding, ${\rm PE}\in\mathbb{R}^{d_{embed}\times\ell}$ is the positional encoding. ${\rm B}_1,\cdots,{\rm B}_{L_T}:\mathbb{R}^{d_{embed}\times\ell}\to\mathbb{R}^{d_{embed}\times\ell}$ are the transformer blocks where each block consists of the residual composition of multi-head attention layers  and feed-forward layers. ${\rm DE}:\mathbb{R}^{d_{embed}\times\ell}\to\mathbb{R}$ is the decoding layer which outputs the first element in the last column.
\end{definition}

In our analysis, we utilize the well-known sinusoidal positional encoding $\mathcal{I}_j\in\mathbb{R}^2$, which can be interpreted as rotations of a unit vector $e_1$ within the first quadrant of the unit circle. More precisely, for an embedding matrix $H = {\rm PE}+{\rm E}(x)$ given in (\ref{Hmatrix}), the first two rows are the data terms, which are used to approximate target function. The third and fourth rows are interaction terms with $\mathcal{I}_j=(\cos(\frac{j\pi}{2\ell}),\sin(\frac{j\pi}{2\ell}))^\top$, determining when each token embedding will interact with another in the attention mechanism, where $\ell$ is the number of hidden tokens. The last (fifth) row are constant terms. 
\begin{definition} 
[Transformer Network Class] \label{TNC}

The transformer network class with weights $\theta$ is
    \begin{align*}
        &\mathcal{T}(L_T, m_T,d_{embed},\ell,L_{\text{FFN}},w_{\text{FFN}},R,\kappa) \\
        &=\Big\{{\rm T}(\theta;\cdot) \text{ }| \text{ } {\rm T}(\theta;\cdot)  \text{ has the form }(\ref{Trep}) \text{ with } L_T \text{ transformer blocks, } 
        \text{at most } 
        m_T\text{ attention heads in} \\ &\quad\quad\quad\quad\quad\quad \text{each block, embedded dimension } d_{embed}, \text{number of hidden tokens } \ell, \text{and } L_{\text{FFN}} 
        \text{ layers} \\
        &\quad\quad\quad\quad\quad\quad \text{of feed-forward networks with}
        \text{ hidden width } w_{\text{FFN}},
         \text{ with output }    \|{\rm T}(\theta;\cdot)\|_{L^{\infty}(\mathbb{R}^D)}\leq R
         \\
    &\quad\quad\quad\quad\quad\quad \text{and weight magnitude } \|\theta\|_{\infty}\leq \kappa\Big\}.
    \end{align*}
\end{definition}


Here $\|\theta\|_{\infty}$ represent the maximum magnitude of the network parameters.
When there is no ambiguity in the context, we will shorten the notation $\mathcal{T}(L_T, m_T,d_{embed},\ell,L_{\rm FFN},w_{\rm FFN},R,\kappa)$ as $\mathcal{T}$. Throughout the paper, we use $x=(x^1,\cdots,x^D)$ as the input variable , with each $x^i$ being the $i$-th component of $x$. We summarize the notations in Table~\ref{tabnotation} in the Appendix~\ref{secAppendixNotation}.

\section{Transformer Approximation and Generalization Theory} \label{secAGTheory}


We next present our main results about  approximation and generalization theories for estimating functions in \eqref{eq:f}.

\subsection{Assumptions}

\begin{assumption} [Manifold] \label{assumpm}
    Let $\mathcal{M}\subseteq [0,1]^D$ be a non-empty, compact, connected $d$-dimensional Riemannian  manifold isometrically embedded in $\mathbb{R}^D$ with a positive reach $\tau_{\mathcal{M}} >0$.
    The tubular region $\mathcal{M}(q)$ defined in (\ref{tube}) satisfies $q\in [0,1)$ and $\mathcal{M}(q)\subseteq [0,1]^D$.
\end{assumption}


\begin{assumption}[Target function]
\label{assumpf}
 The target function $f:\mathcal{M}(q)\to\mathbb{R}$ can be written in \eqref{eq:f} such that  
$f:=g\circ\pi_{\mathcal{M}}$ and $g:\mathcal{M}\to\mathbb{R}$ is  $\alpha$-H\"{o}lder continuous with H\"{o}lder exponent $\alpha\in (0,1]$  and H\"{o}lder constant $L>0$:
\[|g(z)-g(z')|\leq Ld^{\alpha}_{\mathcal{M}}(z,z') \text{ }\text{ }\text{for all} \text{ }\text{ } z,z'\in\mathcal{M}.\]
In addition, we assume $\|f\|_{L^{\infty}(\mathcal{M}(q))}\le R$ for some $R>0$.
\end{assumption}


\subsection{Transformer Approximation Theory}


Our first contribution is a universal approximation theory for functions satisfying Assumption \ref{assumpf} by a transformer network.

\begin{theorem} \label{thmApprox}
Suppose  Assumption~\ref{assumpm} holds. For any $\epsilon\in (0,\min\{1,(\tau_{\mathcal{M}}/2)^{\alpha}\})$, there exists a transformer network ${\rm T}(\theta;\cdot)\in\mathcal{T}(L_{\rm T}, m_{\rm T},d_{embed},\ell,L_{\rm FFN},w_{\rm FFN},R,\kappa)$ with parameters 
\begin{align*}
&L_{\rm T}=O\left(d+\ln\left(\ln(\epsilon^{-1})\right)\right), \text{ } m_{\rm T}=O\left(D\epsilon^{-\frac{d}{\alpha}}(1-q)^{-2d}\right), \text{ } d_{embed}=5, \\
&\ell=O\left(D\epsilon^{-\frac{d}{\alpha}}(1-q)^{-2d}\right), \text{ } L_{\rm FFN}=6, \text{ } w_{\rm FFN}=5, \text{ } \kappa=O\left(D^2\epsilon^{-\frac{2d+8}{\alpha}}(1-q)^{-2d-8}\right) 
\end{align*}
such that, for any $f$ satisfying Assumption~\ref{assumpf}, if the network parameters $\theta$ are properly chosen, the network yields a function  ${\rm T}(\theta;\cdot)$ with 
\begin{equation}
    \|{\rm T}(\theta;\cdot)-f\|_{\mathcal{L}^{\infty}(\mathcal{M}(q))}\leq\epsilon.
\end{equation}
The notation $O(\cdot)$ hides the dependency on $d,q,\tau_{\mathcal{M}},L,R,{\rm Vol}(\mathcal{M})$. Importantly, the $O(\cdot)$ dependency for $L_{\rm T}$ is only on some absolute constants.
\end{theorem}

{The proof of Theorem~\ref{thmApprox} uses the piecewise constant oracle approximator as detailed by \eqref{hatf} in Section \ref{secPfSketch}, which was originally proposed in \citet{Cloninger21}. In this paper, this oracle approximator is realized by using transformers to implement the basic arithmetic operation detailed in Table \ref{tableBAO}.}

The proof of Theorem~\ref{thmApprox} is provided in Section \ref{secPfSketch} and a flow chat of our transformer network is illustrated in Figure \ref{networkfig}. One notable feature of Theorem~\ref{thmApprox} is that the network is shallow. It only requires near constant depth $O(d+\ln(\ln(\epsilon^{-1})))$ to  approximate the function $f$ defined on the noisy manifold with any accuracy $\epsilon$. This highlights a key advantage of Transformers over feed-forward ReLU networks, which require substantially more layers, e.g., $O(\ln(\frac{1}{\epsilon}))$, to achieve the same accuracy \citep{Y17}.

\subsection{Transformer Generalization Theory}

Theorem \ref{thmApprox}  focuses on the existence of  a transformer network class which universally approximates all target functions satisfying Assumption \ref{assumpf}. However, it does not yield a computational strategy to obtain the network parameters for any specific function. In practice, the network parameters are obtained by an empirical risk minimization.

Suppose $\{x_i\}_{i=1}^n$ are $n$ i.i.d samples from a distribution $P$ supported on $\mathcal{M}(q)$, and their corresponding function values are $\{f(x_i)\}_{i=1}^n$. Given $n$ training samples  $\{(x_i,f(x_i))\}_{i=1}^n$, we consider the empirical risk minimizer $\hat{\rm T}_n$ such that 
\begin{equation} \label{hatT}
\textstyle \hat{\rm T}_n:=\argmin_{\rm T\in\mathcal{T}} \frac{1}{n}\sum_{i=1}^n ({\rm T}(x_i)-f(x_i))^2,
\end{equation}
where $\mathcal{T}$ is a transformer network class. The  squared generalization error of $\hat{\rm T}_n$ is 
\begin{equation} \label{sqGenError}
\mathbb{E}\|\hat{\rm T}_n-f\|^2_{L^2(P)} = \mathbb{E}\int_{\mathcal{M}(q)}(\hat{\rm T}_n(x)-f(x))^2dP,
\end{equation}
where the expectation is taken over  $\{x_i\}_{i=1}^n$.

Our next result establishes a generalization error bound for the regression of $f$.

\begin{theorem} \label{thmGeneralization}
Suppose  Assumptions~\ref{assumpm} and \ref{assumpf} hold. Let $\{(x_i, f(x_i))\}_{i=1}^n$  are $n$ training samples where $\{x_i\}_{i=1}^n$ are $n$ i.i.d samples of a distribution $P$ supported on $\mathcal{M}(q)$. 
If the transformer network class  $\mathcal{T}(L_{\rm T}, m_{\rm T},d_{embed},\ell,L_{\rm FFN},w_{\rm FFN},R,\kappa)$ has parameters 
\begin{align*}
&L_{\rm T}=O\left(d+\ln\left(\ln(n^{\frac{\alpha}{2\alpha+d}})\right)\right), \text{ } m_{\rm T}=O \left(Dn^{\frac{d}{2\alpha+d}}(1-q)^{-2d}\right), \text{ } d_{embed}=5, \\
&\ell=O\left(Dn^{\frac{d}{2\alpha+d}}(1-q)^{-2d}\right), \text{ } L_{\rm FFN}=6, \text{ } w_{\rm FFN}=5, \text{ } \kappa=O\left(D^2n^{\frac{2d+8}{2\alpha+d}}(1-q)^{-2d}\right)
\end{align*}
with $O(\cdot)$ hides the dependency on $d,q,\tau_{\mathcal{M}},L,R,{\rm Vol}(\mathcal{M})$. Importantly, the $O(\cdot)$ dependency for $L_{\rm T}$ is only on some absolute constants. Then the empirical risk minimizer $\hat{\rm T}_n$ given by \eqref{hatT} satisfies
\begin{equation}    \mathbb{E}\|\hat{\rm T}_n-f\|^2_{L^2(P)} \leq \tilde{O}\left((1-q)^{-2d}D^2d^3 n^{-\frac{2\alpha}{2\alpha+d}}\right)
\label{eq:generalizationbound}
\end{equation}
where $\tilde{O}(\cdot)$ hides the logarithmic dependency on $D,d,q,n,\alpha,L,R,\tau_{\mathcal{M}},{\rm Vol}(\mathcal{M})$, and {polynomial dependency on ${\rm Vol}(\mathcal{M})$.} 
\end{theorem}

The  proof of Theorem~\ref{thmGeneralization} is provided in Section \ref{secPfSketch}. Theorem \ref{thmGeneralization} shows that the squared generalization error of $\hat{\rm T}$ is upper bounded in the order of $n^{-\frac{2\alpha}{2\alpha+d}}$. In the exact manifold case when $q=0$, Theorem \ref{thmGeneralization} gives rise to the min-max regression error \citep{gyorfi2006distribution}. In the noisy case when $q\in (0,1)$, Theorem \ref{thmGeneralization} demonstrates a denoising phenomenon given by transformers such that when the sample size $n$ increases, the generalization error converges to $0$ at a fast rate depending on the intrinsic dimension $d$.

\begin{table*}[t]
    \centering
    \caption{\small The bound on each parameter in the transformer network class to implement certain operations for input $x=(x^1,\cdots,x^D)\in\mathbb{R}^D$ and $y=(y^1,\cdots,y^D)\in\mathbb{R}^D$. The notation $\odot$ stands for componentwise product and $\circ r$ stands for componentwise $r$-th power. Note that the map $x^1\mapsto \frac{1}{x^1}$ requires $x^1$ bounded above and bounded away from zero if $x^1>0$, and $x^1$ bounded below and bounded away from zero if $x^1<0$. The tolerance for the last operation is measured in $\|\cdot\|_1$ norm while others are measured in $\|\cdot\|_{\infty}$ norm.}
    \label{tableBAO}
    \vspace{2mm}
    \begin{tabular}{lllllll}
    \toprule
       Operations  & $L_{\rm T}$ & $m_{\rm T}$ & $L_{\rm FFN}$ & $w_{\rm FFN}$ &  tolerance & Reference  \\
       \midrule
       $x\mapsto \sum_{i=1}^D x^i$  & $O(1)$ & $O(D)$ & $O(1)$ & $O(1)$ &  0 & Lemma~\ref{lemsumD} \\
         $x\mapsto x+c$ & $O(1)$ & $O(D)$ & $O(1)$ & $O(1)$ & 0 & Lemma~\ref{lemadd}  \\
         $x\mapsto cx$  & $O(1)$ & $O(D)$ & $O(1)$ & $O(1)$ &  0 & Lemma~\ref{lemcmulti} \\
         $x\mapsto x\odot x$  & $O(1)$ & $O(D)$ & $O(1)$ & $O(1)$ &  0 & Lemma~\ref{lemsq} \\
         
         $(x,y)\mapsto x\odot y$  & $O(1)$ & $O(D)$ & $O(1)$ & $O(1)$ &  0 & Lemma~\ref{lempp} \\
         $x\mapsto x^{\circ r}$  & $O(\ln(r))$ & $O(rD)$ &  $O(1)$ & $O(1)$ &  0 & Lemma~\ref{lemrpower} \\
         $x^1\mapsto \frac{1}{x^1}$  & $O(\ln(\ln(\frac{1}{\epsilon})))$ & $O(\ln(\frac{1}{\epsilon}))$ &  $O(1)$ & $O(1)$ &  $\epsilon$ & Lemma~\ref{lemdiv} \\
         $x\mapsto \tilde\eta_i(x)$  & $O(d)$ & $O(D)$ & $O(1)$ & $O(1)$ &  0 & Proposition~\ref{repetatildei} \\
         $x\mapsto (\eta_1(x),\cdots,\eta_K(x))$  & $O(d+\ln(\ln(\frac{1}{\epsilon})))$ & $O(D\epsilon^{-d})$ &  $O(1)$ & $O(1)$ &  $\epsilon$ & Proposition~\ref{approxetai} \\
         \bottomrule
    \end{tabular}   
\end{table*}

\section{Proof of Main Results} \label{secPfSketch}



\subsection{Basic Arithmetic Operations via Transformer}

To prove our main results, let us first construct transformers to implement basic arithmetic operations such as addition, constant multiplication, product, division, etc,.
All the basic arithmetic operations are proved in details in Appendix~\ref{secAppendixBAO}.  The  proofs utilizes the Interaction Lemma~\ref{IAlemma} \citep{Havrilla24}, which states that we can construct an attention head such that one token interacts with exactly another token in the embedding matrix. This allows efficient parallel implementation of these fundamental arithmetic operations (see also Remarks \ref{rmkflexibility} and \ref{rmkparalell} ). {Note that in our construction, each coordinate $x^i$, $1\leq i\leq D$, is stored in a token in the initial embedding matrix $H$.}

For convenience, we summarize all the operations implemented via transformer in Table~\ref{tableBAO}. These basic operations can also serve as building blocks for other tasks of independent interest.
{In literature, the basic operations in Table~\ref{tableBAO} have been implemented or approximated by feedforward neural networks (FNNs) \citep{Y17}. Compared with FNNs, transformers can implement the multiplication operation more efficiently. In Table~\ref{tableBAO}, a transformer network with a constant depth can exactly implement multiplication with zero error. In comparison, to approximate the multiplication operation with $\epsilon$ error, the FNN constructed in \citet{Y17} has depth in the order of $\ln(1/\epsilon)$. 
In Table \ref{tab:transformer_FNN}, we provide the network depth and width comparison between transformer and FNN for two arithmetic operations: sum of componentwise squares and vector dot product.
}

{
To provide additional evidence of the superior expressive power of Transformer over FNN, we perform a set of numerical experiments in Section \ref{sec: expCompare}  to model these two arithmetic operations (e.g., see Table \ref{tableExpCompare}).}

\begin{table*}[t]
    \centering
    \caption{\small Network depth/width/tolerance of Transformer v.s. FNN on different arithmetic operations}
    \label{tab:transformer_FNN}
    \vspace{2mm}
    \begin{tabular}{ccc}
    \toprule
       Arithmetic operations  & Transformer depth / width / tolerance & FNN depth / width / tolerance  \cr
       \midrule
       $x\mapsto \left(\sum_{i=1}^D x^i\right)^2$
         &  $O(1)$ / $O(D)$ / $0$ & $O(\ln(\frac{1}{\epsilon}))$ / $O(1)$ / $\epsilon$  \cr
         \midrule
         $x\cdot y\mapsto \sum_{i=1}^D x^iy^i$ & $O(1)$ / $O(D)$ / $0$ & $O(\ln(\frac{1}{\epsilon}))$ / $O(D)$ / $\epsilon$ \cr
         \bottomrule
    \end{tabular}   
\end{table*}

\begin{lemma} [Sum of Tokens] \label{lemsumD}
    Let $d_{embed}=5$, $M>0$, and $x=(x^1,\cdots,x^D)$ be vector in $\mathbb{R}^D$ such that $\|x\|_{\infty}\leq M$. Let $H$ be an embedding matrix of the form 
    \begin{equation} \label{Hmatrix}
    H = 
\begin{bmatrix}
    x^1 & \cdots & x^D & \mathbf{0}  \\
    0 & \cdots & \cdots & 0 \\
    \mathcal{I}_1 & \cdots & \cdots & \mathcal{I}_{\ell} \\
    1 & \cdots & \cdots & 1
\end{bmatrix}
\in \mathbb{R}^{d_{embed}\times \ell},
\end{equation}
where $\ell\geq D+1$. Then there exists a transformer block $B\in\mathcal{B}(D,6,d_{embed})$ such that
\begin{equation}
    B(H)=
    \begin{bmatrix}
    x^1 & \cdots & x^D & x^1+\cdots+x^D & \mathbf{0} \\
    0 & \cdots & \cdots & \cdots & 0 \\
    \mathcal{I}_1 & \cdots & \cdots & \cdots &  \mathcal{I}_{\ell} \\
    1 & \cdots & \cdots & \cdots & 1
    \end{bmatrix}
\end{equation}
with $\|\theta_B\|_\infty\leq O(\ell^2M^2\|H\|^2_{\infty,\infty})$. We say $B$ implements the sum of tokens in $x$.
\end{lemma}

\begin{lemma} [Constant Addition] \label{lemadd}
    Let $d_{embed}=5$, $M>0$, $c=(c^1,\cdots,c^D)$ and $x=(x^1,\cdots,x^D)$ be vectors in $\mathbb{R}^D$ such that $\|x\|_{\infty}+\|c\|_{\infty}\leq M$. Let $H$ be an embedding matrix of the form 
    \begin{equation*} \label{H1}
    H = 
\begin{bmatrix}
    x^1 & \cdots & x^D & \mathbf{0}  \\
    0 & \cdots & \cdots & 0 \\
    \mathcal{I}_1 & \cdots & \cdots & \mathcal{I}_{\ell} \\
    1 & \cdots & \cdots & 1
\end{bmatrix}
\in \mathbb{R}^{d_{embed}\times \ell},
\end{equation*}
where $\ell\geq 2D$. Then there exists a transformer block $B\in\mathcal{B}(D,6,d_{embed})$ such that
\begin{equation}
    B(H)=
    \begin{bmatrix}
    x^1 & \cdots & x^D & x^1+c^1 & \cdots & x^D+c^D & \mathbf{0} \\
    0 & \cdots & \cdots & \cdots & \cdots & \cdots & 0 \\
    \mathcal{I}_1 & \cdots & \cdots & \cdots & \cdots & \cdots & \mathcal{I}_{\ell} \\
    1 & \cdots & \cdots & \cdots & \cdots & \cdots & 1
    \end{bmatrix}
\end{equation}
with $\|\theta_B\|_\infty\leq O(\ell^2 M^2\|H\|^2_{\infty,\infty})$. We say $B$ implements the addition of $c$ to $x$.
\end{lemma}

\begin{lemma} [Constant Multiplication] \label{lemcmulti}
    Let $M>0$, and $c=(c^1,\cdots,c^D)$ and $x=(x^1,\cdots,x^D)$ be vectors in $\mathbb{R}^D$ such that $\|c\odot x\|_{\infty}\leq M$. Let $H$ be an embedding matrix of the form 
    \begin{equation*} 
    H = 
\begin{bmatrix}
    x^1 & \cdots & x^D & \mathbf{0}  \\
    0 & \cdots & \cdots & 0 \\
    \mathcal{I}_1 & \cdots & \cdots & \mathcal{I}_{\ell} \\
    1 & \cdots & \cdots & 1
\end{bmatrix}
\in \mathbb{R}^{d_{embed}\times \ell},
\end{equation*}
where $\ell\geq 2D$. Then there exists a transformer block $B\in\mathcal{B}(D,6,d_{embed})$ such that
\begin{equation}
    B(H)=
    \begin{bmatrix}
    x^1 & \cdots & x^D & c^1x^1 & \cdots & c^Dx^D & \mathbf{0} \\
    0 & \cdots & \cdots & \cdots & \cdots & \cdots & 0 \\
    \mathcal{I}_1 & \cdots & \cdots & \cdots & \cdots & \cdots & \mathcal{I}_{\ell} \\
    1 & \cdots & \cdots & \cdots & \cdots & \cdots & 1
    \end{bmatrix}.
\end{equation}
with $\|\theta_B\|_\infty\leq O(\ell^2 M^2\|H\|^2_{\infty,\infty})$. We say $B$ implements the multiplication of $c$ to $x$ componentwisely.
\end{lemma}

\begin{lemma} [Squaring] \label{lemsq}
    Let $M>0$, and $x=(x^1,\cdots,x^D)$ be vector in $\mathbb{R}^D$ such that $\|x\|_{\infty}\leq M$. Let $H$ be an embedding matrix of the form 
    \begin{equation*} 
    H = 
\begin{bmatrix}
    x^1 & \cdots & x^D & \mathbf{0}  \\
    0 & \cdots & \cdots & 0 \\
    \mathcal{I}_1 & \cdots & \cdots & \mathcal{I}_{\ell} \\
    1 & \cdots & \cdots & 1
\end{bmatrix}
\in \mathbb{R}^{d_{embed}\times \ell},
\end{equation*}
where $\ell\geq 2D$. Then there exist three transformer blocks $B_1,B_2,B_3\in\mathcal{B}(D,6,d_{embed})$ such that
\begin{equation}
    B_3\circ B_2\circ B_1(H)=
    \begin{bmatrix}
    x^1 & \cdots & x^D & (x^1)^2 & \cdots & (x^D)^2 & \mathbf{0} \\
    0 & \cdots & \cdots & \cdots & \cdots & \cdots & 0 \\
    \mathcal{I}_1 & \cdots & \cdots & \cdots & \cdots & \cdots & \mathcal{I}_{\ell} \\
    1 & \cdots & \cdots & \cdots & \cdots & \cdots & 1
    \end{bmatrix}
\end{equation}
with $\|\theta_B\|_\infty\leq O(\ell^2 M^2\|H\|^2_{\infty,\infty})$. We say $B_1,B_2,B_3$ implements the square of $x$.
\end{lemma}

\begin{lemma} [Componentwise Product] \label{lempp}
     Let $M>0$, $x=(x^1,\cdots,x^D)$ and $y=(y^1,\cdots,y^D)$ be vectors in $\mathbb{R}^D$ be such that $\|x\odot y\|_{\infty}+\|x\|_{\infty}+\|y\|_{\infty}\leq M$. Let $H$ be an embedding matrix of the form 
    \begin{equation} \label{H2}
    H = 
\begin{bmatrix} 
    x^1 & \cdots & x^D & y^1 & \cdots & y^D & \mathbf{0}  \\
    0 & \cdots & \cdots & \cdots  & \cdots &  \cdots  & 0 \\
    \mathcal{I}_1 & \cdots & \cdots & \cdots  &  \cdots & \cdots & \mathcal{I}_{\ell} \\
    1 &  \cdots & \cdots & \cdots & \cdots & \cdots & 1
\end{bmatrix}
\in \mathbb{R}^{d_{embed}\times \ell},
\end{equation}
where $\ell\geq 3D$. Then there exist three transformer blocks $B_1,B_2,B_3\in\mathcal{B}(D,6,d_{embed})$ such that
\begin{equation}
    B_3\circ B_2\circ B_1(H)=
    \begin{bmatrix}
    x^1 & \cdots & x^D & y^1 & \cdots & y^D & x^1y^1 & \cdots & x^Dy^D & \mathbf{0} \\
    0 & \cdots & \cdots & \cdots & \cdots & \cdots & \cdots & \cdots& \cdots & 0 \\
    \mathcal{I}_1 & \cdots & \cdots &\cdots & \cdots &  \cdots & \cdots & \cdots & \cdots & \mathcal{I}_{\ell} \\
    1 & \cdots & \cdots & \cdots & \cdots & \cdots& \cdots & \cdots & \cdots &  1
    \end{bmatrix}
\end{equation}
with $\|\theta_B\|_\infty\leq O(\ell^2 M^2\|H\|^2_{\infty,\infty})$. We say $B_1,B_2,B_3$ implements the componentwise product between $x$ and $y$.
\end{lemma}

\begin{lemma} [Componentwise $r$-th Power] \label{lemrpower}
    Let $M>0$, and $r$ be some integer such that $2^{s-1}<r\leq 2^s$ for some integer $s\geq 1$. Let $x=(x^1,\cdots,x^D)\in\mathbb{R}^D$ such that $ \max_{i,j=1,\cdots,r}\{\|x\|^i_{\infty}+\|x\|^j_{\infty}\}<M$, and $H$ be an embedding matrix of the form 
    \begin{equation*} 
    H = 
\begin{bmatrix}
    x^1 & \cdots & x^D & \mathbf{0}  \\
    0 & \cdots & \cdots & 0 \\
    \mathcal{I}_1 & \cdots & \cdots & \mathcal{I}_{\ell} \\
    1 & \cdots & \cdots & 1
\end{bmatrix}
\in \mathbb{R}^{d_{embed}\times \ell},
\end{equation*}
where $\ell\geq 2^sD$. Then there exists a sequence of transformer blocks $B_i\in\mathcal{B}(2^{\lfloor (i-1)/3 \rfloor}D,6,d_{embed})$, $i=1,\cdots,3s$, such that
\begin{equation}
    B_{3s}\circ B_{3s-1}\circ\cdots\circ B_1(H)=
    \begin{bmatrix}
    x^1 & \cdots & x^D & \cdots & (x^1)^r & \cdots & (x^D)^r & \mathbf{0} \\
    0 & \cdots & \cdots & \cdots & \cdots & \cdots & \cdots & 0 \\
    \mathcal{I}_1 & \cdots & \cdots & \cdots & \cdots & \cdots & \cdots & \mathcal{I}_{\ell} \\
    1 & \cdots & \cdots & \cdots & \cdots & \cdots & \cdots & 1
    \end{bmatrix}\in\mathbb{R}^{d_{embed}\times\ell}
\end{equation}
with $\|\theta_B\|_\infty\leq O(\ell^2 M^2\|H\|^2_{\infty,\infty})$. We say $B_1,\cdots,B_{3s}$ implements the componentwise $r$-th power of $x$.
\end{lemma}

\begin{lemma} [Power Series and Division] \label{lemdiv}
    Let $M>0$, and $r$ be some integer such that $2^{s-1}<r\leq 2^s$ for some integer $s\geq 1$. Let $x=(x^1)\in\mathbb{R}$ such that $\max_{i,j=1,\cdots,r}\{|x|^i+|x|^j\}<M$, and $H$ be an embedding matrix of the form 
    \begin{equation*} 
    H = 
\begin{bmatrix}
    x^1 & 0 & \cdots & 0  \\
    0 & \cdots & \cdots & 0 \\
    \mathcal{I}_1 & \cdots & \cdots & \mathcal{I}_{\ell} \\
    1 & \cdots & \cdots & 1
\end{bmatrix}
\in \mathbb{R}^{d_{embed}\times \ell},
\end{equation*}
where $\ell\geq 2^s$. Then there exists a sequence of transformer blocks $B_i\in\mathcal{B}(2^{\lfloor (i-1)/3 \rfloor},6,d_{embed})$, $i=1,\cdots,3s$, $B_{3s+1}\in\mathcal{B}(r,6,d_{embed})$ such that
\begin{equation*}
B_{3s+1}\circ\cdots\circ B_{1}(H)=
    \begin{bmatrix}
    (x^1)^1 & \cdots &(x^1)^r & \sum_{i=1}^r (x^1)^i & \mathbf{0} \\
    0 & \cdots & \cdots & \cdots & 0\\
    \mathcal{I}_1 & \cdots & \cdots & \cdots  & \mathcal{I}_{\ell} \\
    1 & \cdots &\cdots & \cdots & 1 \\ \end{bmatrix}
\end{equation*}
with $\|\theta_B\|_\infty\leq O(\ell^2 M^2\|H\|^2_{\infty,\infty})$. We say $B_1,\cdots,B_{3s+1}$ implements power series of scalar $x$ up to $r$ term.

Moreover, if $x^1\in [c_1,c_2]$ with $0<c_1<c_2$. Let $c$ be a constant such that $1-cx^1\in(-1,1)$. Then there exists a sequence of transformer blocks
$B_1,B_2,B_{3s+4},B_{3s+5}\in\mathcal{B}(1,6,d_{embed})$, $B_{3s+3}\in\mathcal{B}(r,6,d_{embed})$, and  $B_i\in\mathcal{B}(2^{\lfloor (i-3)/3 \rfloor},6,d_{embed})$, for $i=3,\cdots,3s+2$, such that
\begin{equation*}
B_{3s+5}\circ\cdots\circ B_{1}(H)=
    \begin{bmatrix}
    x^1 & -cx^1 & 1-cx^1 &  \cdots & \sum_{i=0}^r(1-cx^1)^i & c\sum_{i=0}^r(1-cx^1)^i & \mathbf{0}  \\
    0 & \cdots & \cdots & \cdots & \cdots & \cdots & 0\\
    \mathcal{I}_1 & \cdots & \cdots  & \cdots & \cdots &  \cdots & \mathcal{I}_{\ell} \\
    1 & \cdots & \cdots & \cdots & \cdots & \cdots & 1 \\ 
    \end{bmatrix}
\end{equation*}
with $\|\theta_B\|_\infty\leq O(\ell^2 M^2\|H\|^2_{\infty,\infty})$. We say $B_1,\cdots,B_{3s+5}$ approximate the division over $x$ with tolerance $(1-cx^1)^{r+1}/x^1$, i.e.,
\[\left|\frac{1}{x}-c\sum_{i=0}^r(1-cx^1)^i\right|\leq \frac{(1-cx^1)^{r+1}}{x^1}.\]
\end{lemma}

With these basic arithmetic operations, we can prove our main results.

\subsection{Proof of Theorem~\ref{thmApprox}}

We prove Theorem~\ref{thmApprox} in two steps. The first step is to approximate $f$ by a piecewise-constant oracle approximator denoted by $\hat{f}$. The second step is to implement the oracle approximator  $\hat{f}$ by a transformer neural network.

\begin{proof} [Proof of Theorem~\ref{thmApprox}]
    
\vspace{2mm}
\noindent
$\bullet$ {\bf Oracle Approximator}\
\\
\noindent
In this proof, we consider the piecewise constant oracle approximator constructed by \citet{Cloninger21}.
Let $Z=\{z_1,\cdots,z_K\}$ be a maximal separated $\delta$ net of $\mathcal{M}$ with respect to $d_{\mathcal{\mathcal{M}}}$. According to \citet[Lemma 6.1]{Cloninger21}, $K \le 3^d {\rm Vol}(\mathcal{M})d^{\frac{d}{2}}\delta^{-d}$. We define the geodesic ball as $U_i:=\{z\in\mathcal{M}: d_{\mathcal{M}}(z,z_i)\leq\delta\}$.  Then the collection $\{U_i\}_{i=1}^K$ covers $\mathcal{M}$ and the preimages $\{\pi^{-1}_{\mathcal{M}}(U_i)\}_{i=1}^K$ covers the approximation domain $\mathcal{M}(q)$.  

For any partition of unity $\{\eta_i(x)\}_{i=1}^K$ subordinate to  the cover $\{\pi^{-1}_{\mathcal{M}}(U_i)\}_{i=1}^K$, we can decompose $f$ as 
$\textstyle f(x)=\sum_{i=1}^K f(x)\eta_i(x).$
Following the idea in \citet{Cloninger21}, we approximate $f$ by the piecewise-constant function
\begin{equation} \label{hatf}
\textstyle
\hat{f}(x)=\sum_{i=1}^K g(z_i)\eta_i(x).   
\end{equation}  
where each $\eta_i$ is constructed as follows. Let $P(v)\in\mathbb{R}^{D\times d}$ be the matrix containing columnwise orthonormal basis for the tangent space $\mathcal{M}$ at $v$. Let $p:=\frac{1}{2}(1+q)$ and $h:=\frac{6}{1-qp^{-1}}$. Define
\begin{align}
\tilde\eta_i(x) &: 
    =\sigma\left(1-\left(\frac{\|x-z_i\|_2}{p\tau_{\mathcal{M}}(z_i)}\right)^2-\left(\frac{\|P(z_i)^{\top}(x-z_i)\|_2}{h\delta}\right)^2\right)
    \nonumber
    \\
    \eta_i(x) &:={\tilde\eta_i(x)}/{\|\tilde\eta(x)\|_1}
\label{eta}
\end{align}
for $i=1,\ldots,K$, and define the vectors:
\begin{align}
\tilde\eta(x) &=(\tilde\eta_1(x),\cdots,\tilde\eta_K(x)), \nonumber
\\
\eta(x) &=(\eta_1(x),\cdots,\eta_K(x)). \label{eq:eta}
\end{align}
The ellipsoidal regions for $\tilde\eta_i> 0$ are illustrated in Figure~\ref{coveringfig}.
In this construction, $\{\eta_i\}_{i=1}^K$ forms a partition of unity subordinate to  the cover $\{\pi^{-1}_{\mathcal{M}}(U_i)\}_{i=1}^K$ of $\mathcal{M}$.
It is proved in  \citet[Proposition 6.3]{Cloninger21} that  $\{\eta_i\}_{i=1}^K$ satisfies the localization property
\begin{equation} \label{loc}
    \textstyle \sup_{x\in\mathcal{M}(q), \eta_i(x)\neq 0} d_{\mathcal{M}(q)}(x,z_i)\leq O(\delta),
\end{equation}
where $O(\cdot)$ hides the constant term in $q$.
Furthermore, $\|\tilde\eta(x)\|_1$ is uniformly bounded above and bounded away from zero. This property is useful when estimating the depth of transformer network (see  Remark~\ref{rmkdivision}). We then have 
\begin{align*}
    |f(x)-\hat{f}(x)|&=\left|\sum_{i=1}^K g(\pi_{\mathcal{M}}(x))\eta_i(x)-\sum_{i=1}^Kg(z_i)\eta_i(x)\right| \\ &\leq\sum_{i=1}^K\left| g(\pi_{\mathcal{M}}(x))-g(z_i)\right|\eta_i(x) \\
    &\leq L\sum_{i=1}^K d^{\alpha}_{\mathcal{M}}(\pi_{\mathcal{M}}(x),z_i)\eta_i(x) \leq O(\delta^\alpha) 
\end{align*}
where $O(\cdot)$ hides the constant terms in $q$ and $L$.

\begin{figure}[t]
    \centering   \includegraphics[width=0.45\linewidth]{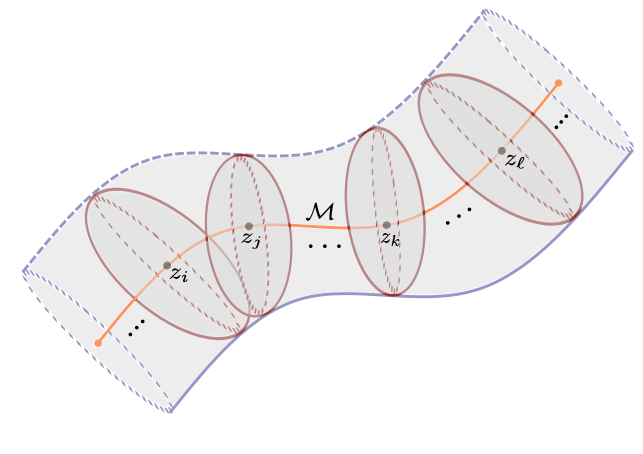}
    \vspace{-8mm}
    \caption{\small The covering of tubular region $\mathcal{M}(q)$, where each ellipsoid represents the region $\{x:\tilde\eta_i(x)> 0\}$.}
    \label{coveringfig}
\end{figure}
\vspace{3mm}
\noindent
$\bullet$ {\bf Implementing the Oracle Approximator by Transformers}\
\\
\noindent
 Since each $\tilde\eta_i(x)$ in (\ref{eta}) is composition of basic arithmetic operations, we can represent it without error by using transformer network. The first result in this subsection establishes the result for representing each $\tilde\eta_i(x)$.

\begin{proposition} \label{repetatildei}
Suppose the Assumption~\ref{assumpm} holds. Let $\{\tilde\eta_i(x)\}_{i=1}^K$ be defined as \eqref{eta}. Then for each fixed $i$, there exists a transformer network ${\rm T}(\theta;\cdot)\in\mathcal{T}(L_T, m_T,d_{embed},\ell,L_{\rm FFN},w_{\rm FFN},R,\kappa)$ with parameters 
\begin{align*}
    &\text{ } L_T=O(d),\text{ } m_T=O(D), \text{ } d_{embed}=5, \text{ } \ell\geq O(D), \\
    &\text{ } L_{\rm FFN}=6,\text{ } w_{\rm FFN}=5, \text{ } \kappa=O(D^2\delta^{-8}) 
\end{align*}
such that 
\begin{equation}
    {\rm T}(\theta;x)=\tilde\eta_i(x)
\end{equation}
for any $x\in [0,1]^D$. The notation $O(\cdot)$ hides the dependency on $d,q,\tau_{\mathcal{M}}$. Importantly, the $O(\cdot)$ dependency for $L_{\rm T}$ is only on some absolute constants.
\end{proposition}

The proof  of Proposition~\ref{repetatildei} is deferred to Appendix~\ref{secAppendixprop12}. The main theme of the proof is that, from (\ref{eta}), it is easy to see that $\tilde\eta_i(x)$ is built from a sequence basic arithmetic operations such as constant addition, constant multiplication, squaring, etc,. Each of these operations is implemented in Table~\ref{tableBAO}. By chaining these operations sequentially, we get the corresponding ${\rm T}(\theta;\cdot)$ to represent $\tilde\eta_i(\cdot)$.

Once each $\tilde\eta_i$ is represented by ${\rm T}(\theta;\cdot)$, we can apply Lemma~\ref{lemdiv} to construct another transformer network which implements $\eta_i(x)=\tilde\eta_i(x)/\|\tilde\eta(x)\|_1$, $i=1,\cdots,K$, and $\eta(x)=(\eta_1(x),\cdots,\eta_K(x))$ within some tolerance. Then take the linear combination of those $\eta_i(x)$ to approximate $\hat{f}$. Note that we need to satisfies $\delta\in (0,\tau_{\mathcal{M}}/2)$ in order to have the cardinality $K=O(\delta^{-d})$ (see Lemma 6.1 in \citep{Cloninger21}), where $O(\cdot)$ hides dependency on $d$ and the volume of manifold ${\rm Vol}(\mathcal{M})$.

The approximation result for $\eta_i(x)$ is presented in Proposition~\ref{approxetai} and its proof is deferred to Appendix~\ref{secAppendixprop12}.

\begin{proposition}
\label{approxetai}
Suppose Assumption~\ref{assumpm} holds. Let $Z=\{z_1,\cdots,z_K\}$ be a maximal separated $\delta$-net of $\mathcal{M}$ with respect to $d_{\mathcal{\mathcal{M}}}$ such that $\delta\in (0,\tau_{\mathcal{M}}/2)$, and define $\eta$ according to \eqref{eq:eta}. 
Then for any $\epsilon\in (0,1)$, there exists ${\rm T}(\theta;\cdot)=({\rm T}^1(\theta;\cdot),\cdots,{\rm T}^K(\theta;\cdot))$ with each ${\rm T}^i(\theta;\cdot)\in\mathcal{T}(L_T, m_T,d_{embed},\ell,L_{\rm FFN},w_{\rm FFN},R,\kappa)$
such that for any $x\in\mathcal{M}(q)$,
\begin{equation}
    |{\rm T}^i(\theta;x) - \eta_i(x)|\leq \epsilon\eta_i(x).
\end{equation}
Consequently, $T(\theta;\cdot)$ satisfies 
\begin{equation}
 \textstyle   \sup_{x\in\mathcal{M}(q)}\|{\rm T}(\theta;x)-\eta(x)\|_1\leq\epsilon.
\end{equation}
The network ${\rm T}(\theta;\cdot)$ has parameters 
\begin{align*}
&L_T=O(d+\ln(\ln(\epsilon^{-1}))), \text{ } m_T=O(D\delta^{-d}), d_{embed}=5, \\
&\ell\geq O(D\delta^{-d}), \text{ } L_{\rm FFN}=6, w_{\rm FFN}=5, \text{ } \kappa=O(D^2\delta^{-2d-8}), 
\end{align*}
where $O(\cdot)$ hides the dependency on $d,q,\tau_{\mathcal{M}},{\rm Vol}(\mathcal{M})$. Importantly, the $O(\cdot)$ dependency for $L_{\rm T}$ is only on some absolute constants.
\end{proposition} 

With Proposition~\ref{approxetai}, we can approximates the $\hat{f}$ in \eqref{hatf} easily by scaling down the tolerance with the supremum norm of $g$.
Let ${\rm T}_1(\theta;\cdot):=({\rm T}_1^1(\theta;\cdot),\cdots,{\rm T}_1^K(\theta;\cdot))$ where each ${\rm T}_1^i$ approximates $\eta_i$ such that 
\begin{equation*}
 \textstyle   \sup_{x\in\mathcal{M}(q)}\|{\rm T}_1(\theta;x)-\eta(x)\|_1\leq {\epsilon}/{\|g\|_{L^\infty (\mathcal{M})}}.
\end{equation*}

Then by Lemma~\ref{lemcmulti} with constant $c=(g(z_1),\cdots g(z_K))$ and Lemma~\ref{lemsumD}, we can construct ${\rm B}_1,{\rm B}_2\in\mathcal{B}(K,6,d_{embed})$ such that ${\rm T}_2:={\rm B}_2\circ {\rm B}_1$ implements the approximation of $\sum_{i=1}^K g(z_i){\rm T}_1^i(\theta;x)$, where ${\rm T}_2$ has $L_{{\rm T}_2}=O(1)$ and $m_{{\rm T}_2}=K=O(\delta^{-d})$. 
Let ${\rm T}:={\rm T}_2\circ {\rm T}_1$, then for any $x\in\mathcal{M}(q)$, we have 
\begin{align*}
\textstyle
|{\rm T}(\theta;x)-\hat{f}(x)|&=
    |\sum_{i=1}^K g(z_i){\rm T}_1^i(\theta;x)-\sum_{i=1}^Kg(z_i)\eta_i(x)| \\
    &\leq \|g\|_{L^\infty(\mathcal{M})}\|{\rm T}_1(\theta;x)-\eta(x)\|_1 =\epsilon.
\end{align*}
An illustration of the constructed transformer network architecture for approximating $\hat{f}$ is provided in Figure~\ref{networkfig}.

\begin{figure*}[t]
    \centering
    \includegraphics[width=1\linewidth]{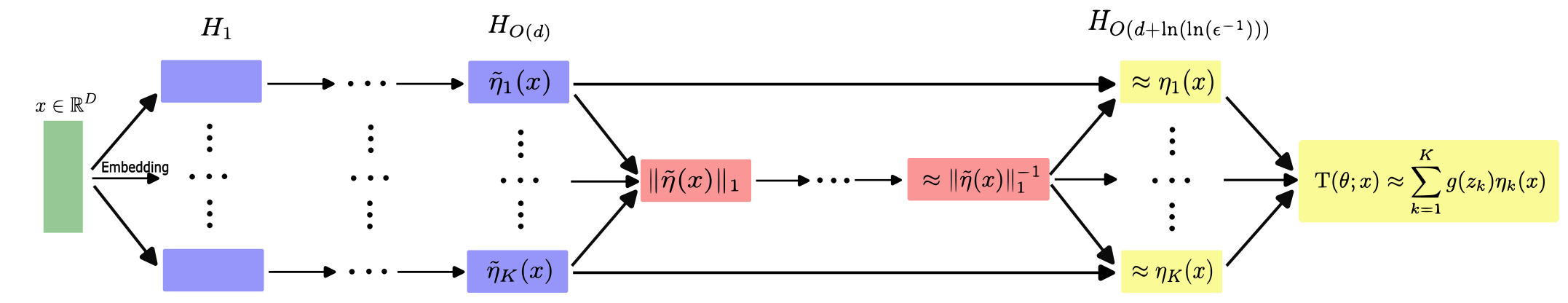}
    \vspace{-3mm}
    \caption{\small Transformer architecture constructed to approximate $\hat{f}$ (the purple component implements each of the $\tilde{\eta}_i$, the red component approximates $\frac{1}{\|\tilde{\eta}\|_1}$, the yellow component approximates each of the $\eta_i(x)$, and then approximates $\hat{f}$).}
    \label{networkfig}
\end{figure*}



\vspace{3mm}
\noindent
$\bullet$ {\bf Putting Error Bounds Together}
\\
\noindent
For any partition of unity $\{\eta_i(x)\}_{i=1}^K$ subordinate to the covering $\{\pi^{-1}_{\mathcal{M}}(U_i)\}_{i=1}^K$, we can write $f(x)=\sum_{i=1}^K f(x)\eta_i(x)$. We consider the following piecewise constant approximation of $f$:
\begin{equation}
    f(x)=\sum_{i=1}^K f(x)\eta_i(x)\approx \hat{f}(x):=\sum_{i=1}^K g(z_i)\eta_i(x),
\end{equation}
By triangle inequality, for any $x\in\mathcal{M}(q)$,
    \begin{align*}
        |f(x)-T(\theta;x)| \leq |f(x)-\hat{f}(x)|+|\hat{f}(x)-T(\theta;x)|.
    \end{align*}
For the first term, we have
\begin{align*}
    |f(x)-\hat{f}(x)|&=\left|\sum_{i=1}^K g(\pi_{\mathcal{M}}(x))\eta_i(x)-\sum_{i=1}^Kg(z_i)\eta_i(x)\right| \leq\sum_{i=1}^K\left| g(\pi_{\mathcal{M}}(x))-g(z_i)\right|\eta_i(x) \\
    &\leq L\sum_{i=1}^K d^{\alpha}_{\mathcal{M}}(\pi_{\mathcal{M}}(x),z_i)\eta_i(x) \leq L\sum_{i=1}^K\left(\frac{72\delta}{(1-q)^2}\right)^\alpha\eta_i(x) =L\left(\frac{72\delta}{(1-q)^2}\right)^\alpha.
\end{align*}
The last equality is due to partition of unity, and the inequality before the last equality is from Proposition 6.3 in \citep{Cloninger21}.

For the second term, by Proposition~\ref{approxetai} and its discussion, we set $\epsilon=\delta^{\alpha}$, and there exists a transformer network $T(\theta;\cdot)\in\mathcal{T}$ with parameters $L_T=O(d+\ln(\ln(\delta^{-\alpha})))$, $m_T=O(D\delta^{-d})$, $d_{embed}=5$, $\ell=O(D\delta^{-d})$, $L_{FFN}=6$, $w_{FFN}=5$, $\kappa=O(D^2\delta^{-2d-8})$, such that 
\begin{equation}
    \|T(\theta;\cdot)-\hat{f}\|_{\mathcal{L}^{\infty}(\mathcal{M}(q))}\leq\delta^\alpha.
\end{equation}

Thus
\begin{align*}
    |T(\theta;x)-f(x)|\leq L\left (\frac{72\delta}{(1-q)^{2}}\right)^\alpha+\delta^{\alpha}=\left(1+L\left (\frac{72}{(1-q)^{2}}\right)^\alpha\right)\delta^\alpha.
\end{align*}

By choosing $\delta$ such that $\left(1+L\left (\frac{72}{(1-q)^{2}}\right)^\alpha\right)\delta^\alpha=\epsilon$, we get $\delta=O(\epsilon^{1/\alpha}(1-q)^2)$ and 
\begin{align*}
    |T(\theta;x)-f(x)|\leq\epsilon.
\end{align*}
Such a transformer network $T(\theta;\cdot)\in\mathcal{T}$ has parameters $L_T=O(d+\ln(\ln(\epsilon^{-1})))$, $m_T=O(D\epsilon^{-\frac{d}{\alpha}}(1-q)^{-2d})$, $d_{embed}=5$, $\ell\geq O(D\epsilon^{-\frac{d}{\alpha}}(1-q)^{-2d})$, $L_{FFN}=6$, $w_{FFN}=5$, $\kappa=O(D^2\epsilon^{-\frac{2d+8}{\alpha}}(1-q)^{-2d-8})$.
\end{proof}

\subsection{Proof of Theorem \ref{thmGeneralization}}


Theorem \ref{thmGeneralization} is proved via a bias-variance decomposition. The bias reflects the approximation error of $f$ by a constructed transformer network, while the variance captures the stochastic error in estimating the parameters of the constructed transformer network. 
For the bias term, we can bound it by using the approximation error bound in Theorem~\ref{thmApprox}. The variance term can be bounded using the covering number of transformers (see Lemma~\ref{coveringlemma}).

\begin{proof} [Proof of Theorem~\ref{thmGeneralization}] By adding and subtracting the twice of the bias term, we can rewrite the squared generalization error as
    \begin{align*}
        \mathbb{E}\|\hat{T}_n-f\|^2_{L^2(P)} &= \mathbb{E}\int_{\mathcal{M}(q)}(\hat{T}_n(x)-f(x))^2dP \\
        & =\mathbb{E}\left[\frac{2}{n}\sum_{i=1}^n (\hat{T}_n(x_i)-f(x_i))^2\right] +\mathbb{E}\int_{\mathcal{M}(q)}(\hat{T}_n(x)-f(x))^2dP - \mathbb{E}\left[\frac{2}{n}\sum_{i=1}^n (\hat{T}_n(x_i)-f(x_i))^2\right].
    \end{align*}
By Jensen's inequality, the bias term satisfies
\begin{align*}
        \mathbb{E}\left[\frac{1}{n}\sum_{i=1}^n (\hat{T}_n(x_i)-f(x_i))^2\right] &= \mathbb{E}\inf_{T\in\mathcal{T}}\left[\frac{1}{n}\sum_{i=1}^n (T(x_i)-f(x_i))^2\right] \leq \inf_{T\in\mathcal{T}}\mathbb{E}\left[\frac{1}{n}\sum_{i=1}^n (T(x_i)-f(x_i))^2\right] \\
        & =\inf_{T\in\mathcal{T}}\int_{\mathcal{M}(q)}(T(x)-f(x))^2dP \leq \inf_{T\in\mathcal{T}}\int_{\mathcal{M}(q)}\|T-f\|^2_{L^\infty(\mathcal{M}(q))}dP \\
        &=\inf_{T\in\mathcal{T}}\|T-f\|^2_{L^\infty(\mathcal{M}(q))} \leq O(\epsilon^2).
    \end{align*}
By Lemma 6 in \citep{Chen22}, the variance term has the bound
\begin{align*}
\mathbb{E}\int_{\mathcal{M}(q)}(\hat{T}_n(x)-f(x))^2dP &- \mathbb{E}\left[\frac{2}{n}\sum_{i=1}^n (\hat{T}_n(x_i)-f(x_i))^2\right]  \\
&\leq\inf_{\delta>0}\left[\frac{104R^2}{3n}\ln\mathcal{N}\left(\frac{\delta}{4R},\mathcal{T},\|\cdot\|_{\infty}\right)+\left(4+\frac{1}{2R}\right)\delta\right] \\
&\leq\left[\frac{104R^2}{3n}\ln\mathcal{N}\left(\frac{1}{4nR},\mathcal{T},\|\cdot\|_{\infty}\right)+\left(4+\frac{1}{2R}\right)\frac{1}{n}\right]
\end{align*}
where $\mathcal{N}\left(\frac{\delta}{4R},\mathcal{T},\|\cdot\|_{\infty}\right)$ is the covering number (defined in Definition~\ref{coveringDef}) of transformer network class $\mathcal{T}$ with $L^{\infty}$ norm. 
By Lemma~\ref{coveringlemma}, we get 
\begin{align*}
&\ln\mathcal{N}\left(\frac{1}{4nR},\mathcal{T},\|\cdot\|_{\infty}\right) \\ 
&\quad\leq \ln\left(2^{L_T+3}nRL_{\text{FFN}}d_{embed}^{18L_T^2}w_{\text{FFN}}^{18L_T^2L_{\text{FFN}}}\kappa^{6L_T^2L_{\text{FFN}}}m_T^{L_T^2}\ell^{L_T^2}\right)^{4d^2_{embed}w^2_{\text{FFN}}D(m_T+L_{\text{FFN}})L_T} \\
&\quad\leq (4d^2_{embed}w^2_{\text{FFN}}D(m_T+L_{\text{FFN}})L_T)(18L^2_TL_{\text{FFN}}\ln(2nRL_{\text{FFN}}d_{embed}w_{\text{FFN}}\kappa m_T\ell)) \\
&\quad\leq 72\ln(2nRL_{\text{FFN}}d_{embed}w_{\text{FFN}}\kappa m_T\ell)d^2_{embed}w^2_{\text{FFN}}Dm_TL^3_TL^2_{\text{FFN}}.
\end{align*}
For target accuracy $\epsilon$, we know from Theorem~\ref{thmApprox} that $L_T=O(d+\ln(\ln(\epsilon^{-1})))$, $m_T=O(D\epsilon^{-\frac{d}{\alpha}}(1-q)^{-2d})$, $d_{embed}=5$, $\ell=O(D\epsilon^{-\frac{d}{\alpha}}(1-q)^{-2d})$, $L_{\rm FFN}=6$, $w_{\rm FFN}=5$, $\kappa=O(D^2\epsilon^{-\frac{2d+8}{\alpha}}(1-q)^{-2d-8})$. This simplifies the above to 
\begin{align*}
\ln\mathcal{N}\left(\frac{1}{4nR},\mathcal{T},\|\cdot\|_{\infty}\right)\leq\tilde{O}\left(D^2d^3\epsilon^{-\frac{d}{\alpha}}(1-q)^{-2d}\right)
\end{align*}
where $\tilde{O}(\cdot)$ hides the logarithmic dependency on $D,d,q,n,\epsilon,\alpha,L,R,\tau_{\mathcal{M}}$,${\rm Vol}(\mathcal{M})$, and polynomial dependency on $d$ and ${\rm Vol}(\mathcal{M})$. Thus, the variance term is bounded by

\begin{align*}
\mathbb{E}\int_{\mathcal{M}(q)}(\hat{T}_n(x)-f(x))^2dP - \mathbb{E}\left[\frac{2}{n}\sum_{i=1}^n (\hat{T}_n(x_i)-f(x_i))^2\right] \leq \tilde{O}\left(\frac{D^2 d^3\epsilon^{-\frac{d}{\alpha}}(1-q)^{-2d}}{n}\right).
\end{align*}
Putting the bias and variance together, we get 

\begin{equation*}
  \mathbb{E}\|\hat{T}_n-f\|^2_{L^2(P)} \leq \tilde{O}\left(\epsilon^2+\frac{D^2 d^3\epsilon^{-\frac{d}{\alpha}}}{n}\right).
\end{equation*}
By balancing the bias and variance, i.e., setting $\epsilon^2=\frac{\epsilon^{-\frac{d}{\alpha}}}{n}$, we get $\epsilon=n^{-\frac{\alpha}{2\alpha+d}}$. This yields
\begin{equation}
  \mathbb{E}\|\hat{T}_n-f\|^2_{L^2(P)} \leq \tilde{O}\left((1-q)^{-2d}D^2d^3 n^{-\frac{2\alpha}{2\alpha+d}}\right)
\end{equation}
as desired.

\end{proof}

\begin{remark}
    It is worth pointing out that the factor of two included in the proof is intended to enhance the rate of convergence of the statistical error.
\end{remark}

\section{Experiments} \label{sec:Exp}

We conduct two experiments: One experiment demonstrates the advantage of transformers over feed-forward neural networks and to support our main results, and the other experiment explores the denoising effect of a pre-trained vision transformer (ViT) \citep{dosovitskiy2021imageworth16x16words}.

\subsection{Transformer v.s. FNN on Learning Arithmetic Operations} \label{sec: expCompare}

{
In this part of the experiments, we consider $x=(x^1,\cdots, x^D)$ and $y=(y^1,\cdots,y^D)$, where the superscripts indicate the coordinate components. We use both transformers and feed-forward neural networks (FNNs) to supervisely learn the following two arithmetic operations: 
\begin{enumerate}
    \item Square of the sum of the components:
    $x\mapsto\left(\sum_{i=1}^D x^i\right)^2$,
    \item Dot product: $x\cdot y\mapsto\sum_{i=1}^D x^iy^i$.
\end{enumerate}
In each pair of experiments, we set the model parameters to be roughly the same and compare the testing errors achieved by both models. }

{For square of the sum of the components operation, we conduct experiments for the following two scenarios: (1) set $D=2$, and $10000$ training samples uniformly sampled from $[0,1]^2$, testing is on $50000$ samples uniformly sampled from $[0,1]^2$. In this case, we set the transformer model with depth $L_{\rm T} = 1$, embedding dimension $d_{embed}=5$, number of attention heads $m_{\rm T} = 5$, feed-forward component dimension $w_{\rm FNN}=5$; and set the FNN model with $2$ layers and each layer is of width $16$.  (2) set $D=4$, and $10000$ training samples uniformly sampled from $[0,1]^4$, testing is on $50000$ samples uniformly sampled from $[0,1]^4$. In this case, we set the transformer model with depth $L_{\rm T} = 1$, embedding dimension $d_{embed}=24$, number of attention heads $m_{\rm T} = 8$, feed-forward component dimension $w_{\rm FNN}=32$; and set the FNN model with $2$ layers and each layer is of width $64$.
}

{For dot product operation, we conduct experiments for the following two scenarios: (1) set $D=4$, and $10000$ training samples uniformly sampled from $[0,1]^4$, testing is on $50000$ samples uniformly sampled from $[0,1]^4$. In this case, we set the transformer model with depth $L_{\rm T} = 1$, embedding dimension $d_{embed}=5$, number of attention heads $m_{\rm T} = 5$, feed-forward component dimension $w_{\rm FNN}=5$; and set the FNN model with $2$ layers and each layer is of width $16$.  (2) set $D=8$, and $10000$ training samples uniformly sampled from $[0,1]^8$, testing is on $50000$ samples uniformly sampled from $[0,1]^8$. In this case, we set the transformer model with depth $L_{\rm T} = 1$, embedding dimension $d_{embed}=24$, number of attention heads $m_{\rm T} = 8$, feed-forward component dimension $w_{\rm FNN}=32$; and set the FNN model with $2$ layers and each layer is of width $64$.
}

{In all experiments, the models are trained on $10000$ i.i.d. samples and tested on $50000$ i.i.d samples with MSE loss and Adam optimizer for $500$ epochs using learning rate $0.001$. The average testing MSEs over 30 repetition are summarized in Table \ref{tableExpCompare}, which clearly demonstrate  the advantages of transformers over FNNs on these basic arithmetic operations.
}

\begin{table*}[t]
    \centering
    \caption{\small Testing MSEs of Transformer v.s. FNN on different arithmetic operations}
    \label{tableExpCompare}
    \vspace{2mm}
    \begin{tabular}{ccccc}
    \toprule
       Arithmetic operations  & $D$ & Transformer MSE & FNN MSE & \# Parameters  \cr
       \midrule
       \multirow{2}{*}{$x\mapsto \left(\sum_{i=1}^D x^i\right)^2$}
         & $2$ &  $9.40\times 10^{-4}$ & $3.12\times 10^{-3}$ &  $\approx 450$  \cr
         & $4$ & $1.04\times 10^{-2}$ & $3.07\times 10^{-2}$  & $\approx 4500$  \cr
         \midrule
         \multirow{2}{*}{$x\cdot y\mapsto \sum_{i=1}^D x^iy^i$} & $4$ & $6.14\times 10^{-3}$ & $1.92\times 10^{-2}$ & $\approx 450$ \cr
          & $8$ & $6.85\times 10^{-3}$ & $3.39\times 10^{-2}$ & $\approx 4500$ \cr
         \bottomrule
    \end{tabular}   
\end{table*}

\subsection{Denoising Effect of Transformers}
Our theoretical results  show that transformers can recover low-dimensional structures even when training data itself may not exactly lie  on a low-dimensional manifold. To validate this findings, we conduct a series of experiments measuring the intrinsic dimension of common computer vision datasets with various levels of isotropic Gaussian noise. We then embed noisy image data using a pre-trained vision transformer (ViT) \citep{dosovitskiy2021imageworth16x16words} and measure the intrinsic dimension of the resulting embeddings.
\vspace{3mm}
\\
\noindent
{\bf Setup.}\ We measure the validation split of Imagenet-1k \citep{imagenet}. We first pre-process images by rescaling to $D=224 \times 244$ dimensions and normalizing pixel values inside of the $[-1,1]^D$ cube. We use the pre-trained \texttt{google/vit-base-patch16-224} model to produce image embeddings of size $196\times768$. To measure intrinsic dimension we use the MLE estimator \citep{MLE} with $K = 30$ neighbors with batch size $4096$ averaged over 50,000 images. We flatten all images beforehand.

\begin{figure}
    \centering
    \begin{tabular}{cc}
        \includegraphics[width=0.45\linewidth]{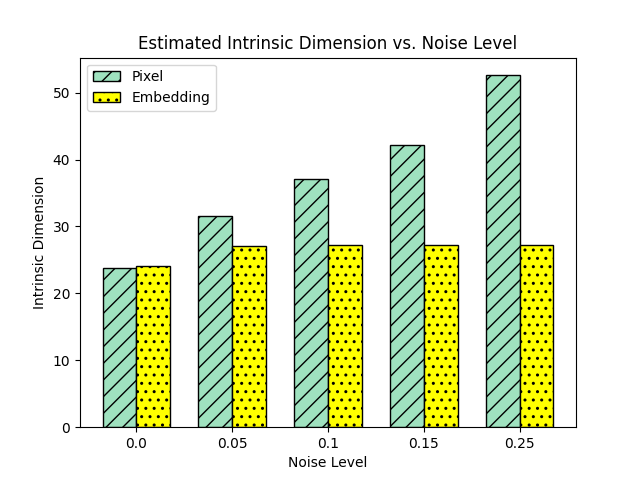} &
\includegraphics[width=0.45\linewidth]{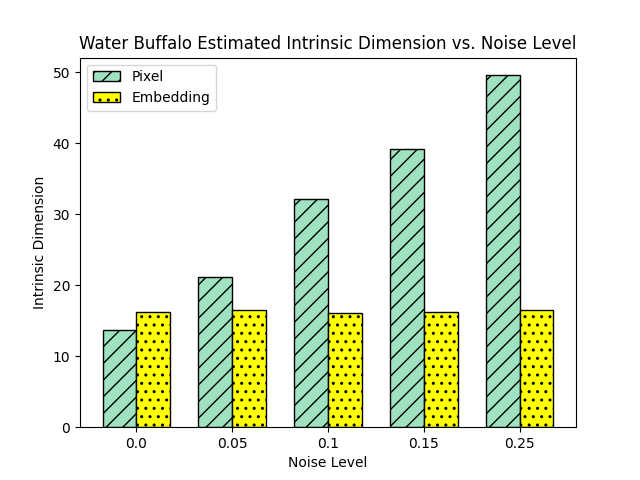}
    \end{tabular}

  \caption{\small Left subplot: Estimated intrinsic dimension (ID) of pixel and embedded image representations with various amounts of isotropic Gaussian noise. Noise added on pixels quickly distorts low-dimensional structures. Embedding with the pre-trained model demonstrates a denoising effect, recovering the original ID at all noise levels. Right subplot: Estimated intrinsic dimension of water buffalo images and embeddings across various noise levels.}
  \vspace{-0.3cm}
    \label{fig:id_noise}
\end{figure}

\vspace{3mm}
\noindent
{\bf Results.}\ Figure~\ref{fig:id_noise} shows that, with no noise, the intrinsic dimensions of this dataset in both pixel and embedding space are measured to be $25$. As isotropic Gaussian noise with increasing variance is added, the intrinsic dimension of pixel data quickly increases. However, the intrinsic dimension of the embedded noisy pixel data remains constant, demonstrating the strong denoising effect of the vision transformer. 
Figure \ref{fig:id_noise} also measures the intrinsic dimension of the water buffalo subset of Imagenet (class $346$) across various noise levels. The estimated image dimension is around $15$ while the estimated embedding dimension is around $18$. However, adding isotropic Gaussian noise quickly increases the intrinsic dimension of images while having a negligible effect on the intrinsic dimension of embeddings.

\vspace{3mm}
\noindent
{The goal of this part of experiments is to support our claim that transformers have a denoising effect when there are noise being added onto the input. It is worth pointing out that classical neural network models such as FNNs may also have this denoising effect \citep{ghorbani2020when}.  We are not claiming that the transformers' representations are robust to noise while FNNs' representations are not.
}

\section{Conclusion and Discussion}
\label{sec:conclusion}

This paper establishes approximation and generalization bounds of transformers for functions which depend on the projection of the input onto a low-dimensional manifold. This regression model is interesting in machine learning applications where the input data contain noise or the function has  low complexity depending on a low-dimensional task manifold. Our theory justifies the capability of transformers in handling noisy data and adapting to low-dimensional structures in the prediction tasks. 
\vspace{3mm}
\\
\noindent
This work considers H\"older functions with H\"older index $\alpha \in (0,1]$. How to estimate this H\"older index in practical applications is an interesting computational problem and {how to extend the theory to more regular functions with $\alpha>1$ is a theoretically interesting problem. Assumption \ref{assumpf} can be extended to higher-order smoothness, and the piecewise constant oracle approximator can be extended to piecewise polynomial approximator. 
However, with our current technique, the rate of convergence may not be improved when the function has higher-order smoothness due to an approximation error in $\pi_{\mathcal{M}}(x)$.  Even if the target function has higher H\"older  regularity index $\alpha>1$,  our current proof technique will give rise to the error as if $\alpha=1$. Another interesting extension is from the ReLU activation to the softmax or more general activations. Our current technique for proving the Interaction Lemma (Lemma \ref{IAlemma}) only works for thresholding type of activation such as ReLU, and does not work for index selection type of activation such as hardmax (argmax) or softmax. The extension of our results from ReLU to softmax is left for future work.}

\section*{Acknowledgments}

Rongjie Lai's research is supported in part by NSF DMS-2401297. Zhaiming Shen, Alex Havrilla and Wenjing Liao are  supported by NSF DMS-2145167 and DOE SC0024348.  Alex Cloninger is supported in part by NSF CISE-2403452.

\bibliographystyle{plainnat}
\bibliography{references}

\newpage
\appendix 
\section{Table of Notations} \label{secAppendixNotation}

Our notations are summarized in Table \ref{tabnotation}.

\begin{table*}[h]
    \centering
    \caption{\small Table of notations}
    \label{tabnotation}
    \vspace{2mm}
    \begin{tabular}{ll}
    \toprule
       Symbol  & Interpretation \\
       \midrule
       $x=(x^1,\cdots,x^D)$ & input variable in $\mathbb{R}^D$ \\
       $\mathcal{M}$ & a compact $d$-dimensional Riemannian manifold $\mathcal{M}$ isometrically embedded in $\mathbb{R}^D$  \\
       \text{Vol}$(\mathcal{M})$ & volume of the manifold $\mathcal{M}$ \\
       Med($\mathcal{M}$) & medial axis of a manifold $\mathcal{M}$ \\
       $\tau_{\mathcal{}}(v)$ & local reach at of $\mathcal{M}$ at $v$ \\
       $\tau_{\mathcal{M}}$ & local reach of $\mathcal{M}$ \\
        $\pi_{\mathcal{M}}(x)$ & projection of $x\in\mathcal{M}(q)$ onto $\mathcal{M}$\\
        
  $P(v)$ & $D\times d$ matrix consists of
 orthonormal basis of the tangent space of $\mathcal{M}$ at $v$. \\     
 $d_{\mathcal{M}}(x,x')$ & geodesic distance between $x$ and $x'$ \\
 $d_{\mathcal{M}(q)}(v,v')$ & tubular geodesic distance between $v$ and $v'$ \\
        
        $\{z_1,\cdots,z_K\}$ & a maximal separated $\delta$-net of $\mathcal{M}$ with respect to $d_{\mathcal{\mathcal{M}}}$ \\
 $H$ & embedding matrix \\
 
 $d_{embed}$ & embedding dimension \\
 ${\rm T}$ & a transformer network \\
 ${\rm B}$ & a transformer block \\
 $L_T$ & number of transformer blocks in ${\rm T}$ \\
 $m_T$ & maximum number of attention heads in each block of ${\rm T}$ \\
 $\ell$ & number of hidden tokens \\
 $\mathcal{I}_j$ & interaction term $(\cos(\frac{j\pi}{2\ell}), \sin(\frac{j\pi}{2\ell}))^{\top}$ \\
 $H_{i,j}$ & the $(i,j)$-th entry of $H$ \\
 $H_{J,:}$ & submatrix of $H$ with rows with row index in $J$ and all the columns \\
 $H_{:,J}$ & submatrix of $H$ with all the rows and columns with column index in $J$ \\
 $x\odot x$ & componentwise product, i.e., $x\odot x=((x^1)^2,\cdots,(x^D)^2)$ \\
 $x^{\circ r}$ & componentwise $r$-th power, i.e., $x^{\circ r}=((x^1)^r,\cdots,(x^D)^r)$ \\
 $\|x\|_1$ & $\ell^1$ norm of a vector $x$ \\
 $\|x\|_{\infty}$ & maximum norm of a vector $x$ \\
 $\|M\|_{\infty,\infty}$ & maximum norm of a matrix $M$ \\
          \bottomrule
    \end{tabular}   
\end{table*}


\section{Implementing Basic Arithmetic Operations by Transformers} \label{secAppendixLemmas}



\subsection{Interaction Lemma, Gating Lemma, and Decrementing Lemma} \label{secAppendixIAandGating}
We first present three lemmas which will be useful when building the arithmetic operations. The first lemma is called Interaction Lemma. 
\begin{lemma} [Interaction Lemma] \label{IAlemma}
   Let $H=[h_t]_{1\leq t\leq\ell}\in\mathbb{R}^{d_{embed}\times\ell}$ be an embedding matrix such that $h_t^{(d_{embed}-2):(d_{embed}-1)}=\mathcal{I}_t$ and $h^{d_{embed}}_t=1$. Fix $1\leq t_1,t_2\leq \ell, 1\leq i\leq d_{embed}$, and $\ell\in\mathbb{N}$. Suppose $d_{embed}\geq 5$ and $\|H\|_{\infty,\infty}<M$ for some $M>0$, and the data kernels $Q^{data}$ (first two rows in the query matrix $Q$) and $K^{data}$ (first two rows in the key matrix $K$) satisfy $\max\{\|Q^{data}\|_{\infty,\infty},\|K^{data}\|_{\infty,\infty}\}\leq\mu$. Then we can construct an attention head $A$ with $\|\theta_A\|_{\infty}=O(d_{embed}^4\mu^2\ell^2M^2)$ such that 
   \[ A(h_t) = 
   \begin{cases} 
    \sigma(\langle Q^{data}h_{t}, K^{data}h_{t_2}\rangle)e_i & \text{if}\text{  } t=t_1, \\
      0 & \text{otherwise}. \\
   \end{cases}
\]
\end{lemma}

\begin{proof}
    We refer its proof to Lemma 3 in \citep{Havrilla24}.
\end{proof}

\begin{remark}
    The significance of the Interaction Lemma is that we can find an attention head such that one token interacts with exactly another token in the embedding matrix. This property facilitates the flexible implementation of fundamental arithmetic operations, such as addition, multiplication, squaring, etc., while also supporting efficient parallelization.
\end{remark}

The next two lemmas show the way to zero out or subtract off constant from contiguous tokens in the embedding matrix while keep other tokens unchanged via a feed-forward network. 

\begin{lemma} [Gating Lemma] \label{Gatinglemma}
Let $d_{embed}\geq 5$ and $H=[h_t]_{1\leq t\leq\ell}\in\mathbb{R}^{d_{embed}\times\ell}$, be an embedding matrix such that $h_t^{(d_{embed}-2):(d_{embed}-1)}=(\mathcal{I}_t^1,\mathcal{I}_t^2)=\mathcal{I}_t$ and $h^{d_{embed}}_t=1$. Then for any $r_1$ and $r_2$ with $1\leq r_1\leq r_2\leq d_{embed}-3$ and any $k_1,k_2$ with $1\leq k_1,k_2\leq \ell$, there exist both two-layer feed-forward networks $({\rm FFN})$ such that
\begin{equation} \label{eq:gate1}
     \text{\rm FFN}_1(h_t) = 
   \begin{cases} 
    h_t & \text{if}\text{  } t\in\{1,\cdots,k_1\} \\
      \begin{bmatrix}

      (h_t)_{1} \\
          \vdots \\
          (h_t)_{r_1-1} \\ 
          \mathbf{0} \\
          (h_t)_{r_2+1} \\
          \vdots \\
          (h_t)_{d_{embed}-3} \\ 
          \mathcal{I}_t^1 \\
          \mathcal{I}_t^2 \\
          1
      \end{bmatrix} & \text{otherwise}\text{ }  \\
   \end{cases}
\end{equation}
and
\begin{equation} \label{eq:gate2}
    \text{\rm FFN}_2(h_t) = 
   \begin{cases} 
 h_t & \text{if}\text{  } t\in\{k_2,\cdots,\ell\} \\
      \begin{bmatrix}
      (h_t)_{1} \\
          \vdots \\
          (h_t)_{r_1-1} \\
          \mathbf{0} \\
          (h_t)_{r_2+1} \\
          \vdots \\
          (h_t)_{d_{embed}-3} \\
          \mathcal{I}_t^1 \\
          \mathcal{I}_t^2 \\
          1
      \end{bmatrix} & \text{otherwise}\text{ }  \\
   \end{cases}
\end{equation}
Additionally, we have $\|\theta_{\rm FFN}\|_{\infty}\leq O(\ell\|H\|_{\infty,\infty})$.
\end{lemma}

\begin{proof}
    We refer its proof to Lemma 6 in \citep{SHLL25}.
\end{proof}


\begin{lemma} [Decrementing Lemma] \label{lemmadecrement}
Let $d_{embed}\geq 5$ and $H=[h_t]_{1\leq t\leq\ell}\in\mathbb{R}^{d_{embed}\times\ell}$, be an embedding matrix such that $h_t^{(d_{embed}-2):(d_{embed}-1)}=(\mathcal{I}_t^1,\mathcal{I}_t^2)=\mathcal{I}_t$ and $h^{d_{embed}}_t=1$. Then for any $r_1,r_2$ with  $1\leq r_1\leq r_2\leq d_{embed}-3$ and any $k_1,k_2$ with $1\leq k_1,k_2\leq \ell$ and any $M>0$, there exists a six-layer residual feed-forward network $({\rm FFN})$ such that
\[ \text{\rm FFN}(h_t) + h_t = 
   \begin{cases} 
    h_t & \text{if}\text{  } t\in\{1,\cdots,k_1\}\cup\{k_2,\cdots,\ell\} \\
      \begin{bmatrix}
      (h_t)_{1} \\
          \vdots \\
          (h_t)_{r_1-1} \\
          (h_t)_{r_1}-M \\
          \vdots \\
          (h_t)_{r_2}-M \\
          (h_t)_{r_2+1} \\
          \vdots \\
          (h_t)_{d_{embed}-3} \\
          \mathcal{I}_t \\
          1
      \end{bmatrix} & \text{otherwise}\text{ }  \\
   \end{cases}
\]
Additionally, we have $\|\theta_{\rm FFN}\|_{\infty}\leq O(\ell M)$.
\end{lemma}

\begin{proof}
    We refer its proof to Lemma 7 in \citep{SHLL25}.
\end{proof}


\subsection{Proof of Basic Arithmetic Operations} \label{secAppendixBAO}

\subsubsection{Proof of Lemma~\ref{lemsumD}}
\begin{proof} [Proof of Lemma~\ref{lemsumD}]
Let us define each attention head $A_i$, $1\leq i\leq D$, with the data kernel in the form
\begin{equation*}
Q_i^{data}
=
\begin{bmatrix}
    0 & 0 & 0 & 0  & 1 \\
    0 & 0 & 0 & 0  & 1
\end{bmatrix}
\quad 
K_i^{data}
=
\begin{bmatrix}
    1 & 0 & 0 & 0  & 0 \\
    0 & 0 & 0 & 0  & M
\end{bmatrix}.
\end{equation*}
Let $h_{i}$ denote the $i$-th column of $H$, $1\leq i\leq \ell$.
By Lemma \ref{IAlemma}, we can construct $A_i$, $1\leq i\leq D$, such that $h_{D+1}$ interacts with $h_{i}$ only, i.e.,
\begin{align*}
    A_i(h_{D+1}) = \sigma(\langle Q_i^{data}h_{D+1}, K_i^{data}h_{i} \rangle)e_1 
    = \sigma(x^i+M)e_1=(x^{i}+M) e_1,
\end{align*}
and $A_i(h_{t})=0$ when $t\neq D+1$. Then the residual multi-head attention yields 

\begin{align*}
    \text{MHA}(H) + H  =
    \begin{bmatrix}
    x^1 & \cdots & x^D & x^1+\cdots+x^D+DM & \mathbf{0} \\
    0 & \cdots & \cdots & \cdots & 0 \\
    \mathcal{I}_1 & \cdots & \cdots & \cdots &  \mathcal{I}_{\ell} \\
    1 & \cdots & \cdots & \cdots & 1
    \end{bmatrix}.
\end{align*}
Then we apply Lemma \ref{lemmadecrement} to have a $\mathcal{FFN}(6)$ to subtract off the constant $DM$ in the $(D+1)$-th column.  Thus
\begin{equation*}
    B(H)=
    \begin{bmatrix}
    x^1 & \cdots & x^D & x^1+\cdots+x^D & \mathbf{0} \\
    0 & \cdots & \cdots & \cdots & 0 \\
    \mathcal{I}_1 & \cdots & \cdots & \cdots &  \mathcal{I}_{\ell} \\
    1 & \cdots & \cdots & \cdots & 1
    \end{bmatrix}
\end{equation*}
as desired. The weights $\|\theta_B\|_\infty\leq O(\ell^2M^2\|H\|^2_{\infty,\infty})$ follows from Lemma~\ref{IAlemma}.
\end{proof}

\begin{remark} \label{rmkflexibility}
   By reexamining the proof, it is easy to see that the summation term $x^1+\cdots+x^D$ can be put in any column of the first row, not necessarily the $D+1$-th column. This provides a lot of flexibility when parallelizing different basic operations in one transformer block. 
\end{remark}

\subsubsection{Proof of Lemma~\ref{lemadd}}

\begin{proof} [Proof of Lemma~\ref{lemadd}]
Let us define the each attention head $A_i$, $1\leq i\leq D$, with the data kernel in the form
\begin{equation*}
Q_i^{data}
=
\begin{bmatrix}
    0 & 0 & 0 & 0  & 0 \\
    0 & 0 & 0 & 0  & 1
\end{bmatrix}
\quad 
K_i^{data}
=
\begin{bmatrix}
    0 & 0 & 0 & 0  & 0 \\
    1 & 0 & 0 & 0  & c^i+M
\end{bmatrix}.
\end{equation*}
By Lemma \ref{IAlemma}, we can construct $A_i$ such that $h_{D+i}$ interacts with $h_i$ only, i.e.,
\begin{equation*}
    A_i(h_{D+i}) = \sigma(\langle Q_i^{data}h_{D+i}, K_i^{data}h_i \rangle)e_1 = \sigma(x^{i}+c^i+M)e_1=(x^{i}+c^i+M)e_1,
\end{equation*}
and $A_i(h_t)=0$ when $t\neq D+i$. Then the residual multi-head attention yields 

\begin{equation*}
    \text{MHA}(H) + H =
    \begin{bmatrix}
    x^1 & \cdots & x^D & x^1+c^1+M & \cdots & x^D+c^D+M & \mathbf{0} \\
    0 & \cdots & \cdots & \cdots & \cdots & \cdots & 0 \\
    \mathcal{I}_1 & \cdots & \cdots & \cdots & \cdots & \cdots & \mathcal{I}_{\ell} \\
    1 & \cdots & \cdots & \cdots & \cdots & \cdots & 1
    \end{bmatrix}.
\end{equation*}
Then we apply Lemma \ref{lemmadecrement} to have a $\mathcal{FFN}(6)$ to subtract off the constant $M$ only from columns $D+1$ to $2D$. 
Therefore, we have

\begin{equation*}
    B(H)=
    \begin{bmatrix}
    x^1 & \cdots & x^D & x^1+c^1 & \cdots & x^D+c^D & \mathbf{0} \\
    0 & \cdots & \cdots & \cdots & \cdots & \cdots & 0 \\
    \mathcal{I}_1 & \cdots & \cdots & \cdots & \cdots & \cdots & \mathcal{I}_{\ell} \\
    1 & \cdots & \cdots & \cdots & \cdots & \cdots & 1
    \end{bmatrix}
\end{equation*}
as desired. The weights $\|\theta_B\|_\infty\leq O(\ell^2M^2\|H\|^2_{\infty,\infty})$ follows from Lemma~\ref{IAlemma}.
\end{proof}



\subsubsection{Proof of Lemma~\ref{lemcmulti}}
\begin{proof} [Proof of Lemma~\ref{lemcmulti}]
Let us define the each attention head $A_i$, $1\leq i\leq D$, with the data kernel in the form
\begin{equation*}
Q_i^{data}
=
\begin{bmatrix}
    0 & 0 & 0 & 0  & c^i \\
    0 & 0 & 0 & 0  & 1
\end{bmatrix}
\quad 
K_i^{data}
=
\begin{bmatrix}
    1 & 0 & 0 & 0  & 0 \\
    0 & 0 & 0 & 0  & M
\end{bmatrix}.
\end{equation*}
Then by Lemma \ref{IAlemma}, we can construct $A_i$ such that $h_{D+i}$ interacts with $h_i$ only, i.e.,
\begin{equation*}
    A_i(h_{D+i}) = \sigma(\langle Q_i^{data}h_{D+i}, K_i^{data}h_i \rangle)e_1 = \sigma(c^ix^{i}+M)e_1=(c^ix^{i}+M)e_1,
\end{equation*}
and $A_i(h_t)=0$ when $t\neq D+i$. Then the residual multi-head attention yields 

\begin{equation*}
    \text{MHA}(H) + H =
    \begin{bmatrix}
    x^1 & \cdots & x^D & c^1x^1+M & \cdots & c^Dx^D+M & \mathbf{0} \\
    0 & \cdots & \cdots & \cdots & \cdots & \cdots & 0 \\
    \mathcal{I}_1 & \cdots & \cdots & \cdots & \cdots & \cdots & \mathcal{I}_{\ell} \\
    1 & \cdots & \cdots & \cdots & \cdots & \cdots & 1
    \end{bmatrix}.
\end{equation*}
Then we apply Lemma \ref{lemmadecrement} to have a $\mathcal{FFN}(6)$ to subtract off the constant $M$ only from columns $D+1$ to $2D$. Thus
\begin{equation*}
    B(H)=
    \begin{bmatrix}
    x^1 & \cdots & x^D & c^1x^1 & \cdots & c^Dx^D & \mathbf{0} \\
    0 & \cdots & \cdots & \cdots & \cdots & \cdots & 0 \\
    \mathcal{I}_1 & \cdots & \cdots & \cdots & \cdots & \cdots & \mathcal{I}_{\ell} \\
    1 & \cdots & \cdots & \cdots & \cdots & \cdots & 1
    \end{bmatrix}.
\end{equation*}
as desired. The weights $\|\theta_B\|_\infty\leq O(\ell^2M^2\|H\|^2_{\infty,\infty})$ follows from Lemma~\ref{IAlemma}.
\end{proof}

\subsubsection{Proof of Lemma~\ref{lemsq}}

\begin{proof} [Proof of Lemma~\ref{lemsq}]
First, applying Lemma~\ref{lemcmulti} with multiplication constant $c=(1,\cdots,1)$, we can construct the transformer block $B_1\in\mathcal{B}(D,6,d_{embed})$ so that it copies the first $D$ elements in the first row from columns $1,\cdots,D$ to columns $D+1,\cdots,2D$, i.e.,
\begin{equation*}
    H_1: = B_1(H)=
    \begin{bmatrix}
    x^1 & \cdots & x^D & x^1 & \cdots & x^D & \mathbf{0} \\
    0 & \cdots & \cdots & \cdots & \cdots & \cdots & 0 \\
    \mathcal{I}_1 & \cdots & \cdots & \cdots & \cdots & \cdots & \mathcal{I}_{\ell} \\
    1 & \cdots & \cdots & \cdots & \cdots & \cdots & 1
    \end{bmatrix}.
\end{equation*}
For $B_2$, let us define each attention head $A_i$, $1\leq i\leq D$, with the data kernel in the form
\begin{equation*}
Q_i^{data}
=
\begin{bmatrix}
    1 & 0 & 0 & 0  & 0 \\
    0 & 0 & 0 & 0  & 0
\end{bmatrix}
\quad 
K_i^{data}
=
\begin{bmatrix}
    1 & 0 & 0 & 0  & 0 \\
    0 & 0 & 0 & 0  & 0
\end{bmatrix}.
\end{equation*}
Let $h_{1,i}$ denote the $i$-th column of $H_1$, $1\leq i\leq \ell$.
By Lemma \ref{IAlemma}, we can construct $A_i$, $1\leq i\leq D$, such that $h_{1,D+i}$ interacts with $h_{1,i}$ only, i.e.,
\begin{equation*}
    A_i(h_{1,D+i}) = \sigma(\langle Q_i^{data}h_{1,D+i}, K_i^{data}h_{1,i} \rangle)e_1 = \sigma((x^{i})^2)e_1=(x^{i})^2 e_1,
\end{equation*}
and $A_i(h_{1,t})=0$ when $t\neq D+i$. Then the residual multi-head attention yields 

\begin{equation*}
    \text{MHA}(H_1) + H_1 =
    \begin{bmatrix}
    x^1 & \cdots & x^D & (x^1)^2+x^1 & \cdots & (x^D)^2+x^D & \mathbf{0} \\
    0 & \cdots & \cdots & \cdots & \cdots & \cdots & 0 \\
    \mathcal{I}_1 & \cdots & \cdots & \cdots & \cdots & \cdots & \mathcal{I}_{\ell} \\
    1 & \cdots & \cdots & \cdots & \cdots & \cdots & 1
    \end{bmatrix}.
\end{equation*}

Let $H_2:=B_2(H_1)=\text{MHA}(H_1)+H_1$, and we use $h_{2,i}$ to denote the $i$-th column of $H_2$, $1\leq i\leq\ell$. Now again by Lemma~\ref{lemcmulti} with multiplication constant $c=(-1,\cdots,-1)$, we can construct $B_3\in\mathcal{B}(D,6,d_{embed})$ with each attention head $\tilde{A}_i$, $1\leq i\leq D$, such that $h_{2,D+i}$ interacts with $h_{2,i}$ only. Let the data kernel of each $\tilde{A}_i$ in the form 
\begin{equation*}
Q_i^{data}
=
\begin{bmatrix}
    0 & 0 & 0 & 0  & -1 \\
    0 & 0 & 0 & 0  & 1
\end{bmatrix}
\quad 
K_i^{data}
=
\begin{bmatrix}
    1 & 0 & 0 & 0  & 0 \\
    0 & 0 & 0 & 0  & M
\end{bmatrix}.
\end{equation*}
By Lemma \ref{IAlemma}, we have 
\begin{equation*}
    \tilde{A}_i(h_{2,D+i}) = \sigma(\langle Q_i^{data}h_{2,D+i}, K_i^{data}h_{2,i} \rangle)e_1 = \sigma(-x^i+M)e_1=(-x^{i}+M) e_1,
\end{equation*}
and $\tilde{A}_i(h_{2,t})=0$ when $t\neq D+i$. Thus, the residual multi-head attention yields 

\begin{equation*}
    \text{MHA}(H_2) + H_2 =
    \begin{bmatrix}
    x^1 & \cdots & x^D & (x^1)^2+M & \cdots & (x^D)^2+M & \mathbf{0} \\
    0 & \cdots & \cdots & \cdots & \cdots & \cdots & 0 \\
    \mathcal{I}_1 & \cdots & \cdots & \cdots & \cdots & \cdots & \mathcal{I}_{\ell} \\
    1 & \cdots & \cdots & \cdots & \cdots & \cdots & 1
    \end{bmatrix}.
\end{equation*}

Then we apply Lemma \ref{lemmadecrement} to have a $\mathcal{FFN}(6)$ to subtract off the constant $M$ only from columns $D+1$ to $2D$.  Therefore, we have

\begin{equation*}
    B_3\circ B_2\circ B_1(H)=B_3(H_2) = 
    \begin{bmatrix}
    x^1 & \cdots & x^D & (x^1)^2 & \cdots & (x^D)^2 & \mathbf{0} \\
    0 & \cdots & \cdots & \cdots & \cdots & \cdots & 0 \\
    \mathcal{I}_1 & \cdots & \cdots & \cdots & \cdots & \cdots & \mathcal{I}_{\ell} \\
    1 & \cdots & \cdots & \cdots & \cdots & \cdots & 1
    \end{bmatrix}
\end{equation*}
as desired. The weights $\|\theta_B\|_\infty\leq O(\ell^2M^2\|H\|^2_{\infty,\infty})$ follows from Lemma~\ref{IAlemma}.
\end{proof}



\subsubsection{Proof of Lemma~\ref{lempp}}

\begin{proof} [Proof of Lemma~\ref{lempp}]
First, applying Lemma~\ref{lemcmulti} with multiplication constant $c=(1,\cdots,1)$, we can construct the transformer block $B_1\in\mathcal{B}(D,6,d_{embed})$ so that it copies the first $D$ elements in the first row from columns $1,\cdots,D$ to columns $2D+1,\cdots,3D$, i.e.,
\begin{equation*}
    H_1: = B_1(H)=
    \begin{bmatrix}
    x^1 & \cdots & x^D & y^1 & \cdots & y^D & x^1 & \cdots & x^D & \mathbf{0} \\
    0 & \cdots & \cdots & \cdots & \cdots & \cdots & \cdots & \cdots& \cdots & 0 \\
    \mathcal{I}_1 & \cdots & \cdots &\cdots & \cdots &  \cdots & \cdots & \cdots & \cdots & \mathcal{I}_{\ell} \\
    1 & \cdots & \cdots & \cdots & \cdots & \cdots& \cdots & \cdots & \cdots &  1
    \end{bmatrix}.
\end{equation*}

For $B_2$, let us define the each attention head $A_i$, $1\leq i\leq D$, with the data kernel in the form
\begin{equation} \label{dataQKpm}
Q_i^{data}
=
\begin{bmatrix}
    1 & 0 & 0 & 0  & 0 \\
    0 & 0 & 0 & 0  & 1
\end{bmatrix}
\quad 
K_i^{data}
=
\begin{bmatrix}
    1 & 0 & 0 & 0  & 0 \\
    0 & 0 & 0 & 0  & M
\end{bmatrix}.
\end{equation}
By Lemma \ref{IAlemma}, we can construct $A_i$, $1\leq i\leq D$, such that $h_{1,2D+i}$ interacts with $h_{1,D+i}$ only, i.e.,
\begin{equation*}
    A_i(h_{1,2D+i}) = \sigma(\langle Q_i^{data}h_{1,2D+i}, K_i^{data}h_{1,D+i} \rangle)e_1 = \sigma(x^{i}y^i+M)e_1=(x^{i}y^{i}+M) e_1,
\end{equation*}
and $A_i(h_{1,t})=0$ when $t\neq 2D+i$.  Then the residual multi-head attention yields 

\begin{equation*}
    \text{MAH}(H_1) + H_1 =
    \begin{bmatrix}
    x^1 & \cdots & x^D & y^1 & \cdots & y^D & x^1y^1+x^1+M & \cdots & x^Dy^D+x^D+M & \mathbf{0} \\
    0 & \cdots & \cdots & \cdots & \cdots & \cdots & \cdots & \cdots& \cdots & 0 \\
    \mathcal{I}_1 & \cdots & \cdots &\cdots & \cdots &  \cdots & \cdots & \cdots & \cdots & \mathcal{I}_{\ell} \\
    1 & \cdots & \cdots & \cdots & \cdots & \cdots& \cdots & \cdots & \cdots &  1
    \end{bmatrix}.
\end{equation*}

Then we apply Lemma \ref{lemmadecrement} to have a $\mathcal{FFN}(6)$ to subtract off the constant $M$ only from columns $2D+1$ to $3D$. Thus, we have

\begin{equation*}
    H_2:=B_2\circ B_1(H)= B_2(H_1)=
    \begin{bmatrix}
    x^1 & \cdots & x^D & y^1 & \cdots & y^D & x^1y^1+x^1 & \cdots & x^Dy^D+x^D & \mathbf{0} \\
    0 & \cdots & \cdots & \cdots & \cdots & \cdots & \cdots & \cdots& \cdots & 0 \\
    \mathcal{I}_1 & \cdots & \cdots &\cdots & \cdots &  \cdots & \cdots & \cdots & \cdots & \mathcal{I}_{\ell} \\
    1 & \cdots & \cdots & \cdots & \cdots & \cdots& \cdots & \cdots & \cdots &  1
    \end{bmatrix}.
\end{equation*}

Now again by Lemma~\ref{lemcmulti} with multiplication constant $c=(-1,\cdots,-1)$, we can construct $B_3\in\mathcal{B}(D,6,d_{embed})$ with each attention head $\tilde{A}_i$, $1\leq i\leq D$, such that $h_{2,2D+i}$ interacts with $h_{2,i}$ only. Let the data kernel of each $\tilde{A}_i$ in the form 
\begin{equation*}
Q_i^{data}
=
\begin{bmatrix}
    0 & 0 & 0 & 0  & -1 \\
    0 & 0 & 0 & 0  & 1
\end{bmatrix}
\quad 
K_i^{data}
=
\begin{bmatrix}
    1 & 0 & 0 & 0  & 0 \\
    0 & 0 & 0 & 0  & M
\end{bmatrix}.
\end{equation*}
By Lemma \ref{IAlemma}, we have 
\begin{equation*}
    \tilde{A}_i(h_{2,2D+i}) = \sigma(\langle Q_i^{data}h_{2,2D+i}, K_i^{data}h_{2,i} \rangle)e_1 = \sigma(-x^i+M)e_1=(-x^{i}+M) e_1,
\end{equation*}
and $\tilde{A}_i(h_{2,t})=0$ when $t\neq 2D+i$. Thus, the residual multi-head attention yields 

\begin{equation*}
    \text{MHA}(H_2) + H_2 =
    \begin{bmatrix}
    x^1 & \cdots & x^D & y^1 & \cdots & y^D & x^1y^1+M & \cdots & x^Dy^D+M & \mathbf{0} \\
    0 & \cdots & \cdots & \cdots & \cdots & \cdots & \cdots & \cdots& \cdots & 0 \\
    \mathcal{I}_1 & \cdots & \cdots &\cdots & \cdots &  \cdots & \cdots & \cdots & \cdots & \mathcal{I}_{\ell} \\
    1 & \cdots & \cdots & \cdots & \cdots & \cdots& \cdots & \cdots & \cdots &  1
    \end{bmatrix}
\end{equation*}

Then we apply Lemma \ref{lemmadecrement} to have a $\mathcal{FFN}(6)$ to subtract off the constant $M$ only from columns $2D+1$ to $3D$. Therefore, we have

\begin{equation*}
    B_3\circ B_2\circ B_1(H)=B_3(H_2) = 
    \begin{bmatrix}
    x^1 & \cdots & x^D & y^1 & \cdots & y^D & x^1y^1 & \cdots & x^Dy^D & \mathbf{0} \\
    0 & \cdots & \cdots & \cdots & \cdots & \cdots & \cdots & \cdots& \cdots & 0 \\
    \mathcal{I}_1 & \cdots & \cdots &\cdots & \cdots &  \cdots & \cdots & \cdots & \cdots & \mathcal{I}_{\ell} \\
    1 & \cdots & \cdots & \cdots & \cdots & \cdots& \cdots & \cdots & \cdots &  1
    \end{bmatrix}
\end{equation*}
as desired. The weights $\|\theta_B\|_\infty\leq O(\ell^2M^2\|H\|^2_{\infty,\infty})$ follows from Lemma \ref{IAlemma}.
\end{proof}

\subsubsection{Proof of Lemma~\ref{lemrpower}}

\begin{proof} [Proof of Lemma~\ref{lemrpower}]
It suffices to show for the case $r=2^s$.
Let us proceed by induction on $s$. First, suppose $B_1,B_2,B_3\in\mathcal{B}(D,6,d_{embed})$ implements the squaring operation as shown in Lemma~\ref{lemsq}, i.e.,
    \begin{equation*}
    H_3: = B_3\circ B_2\circ B_1(H)=
    \begin{bmatrix}
    x^1 & \cdots & x^D & (x^1)^2 & \cdots & (x^D)^2 & \mathbf{0} \\
    0 & \cdots & \cdots & \cdots & \cdots & \cdots & 0 \\
    \mathcal{I}_1 & \cdots & \cdots & \cdots & \cdots & \cdots & \mathcal{I}_{\ell} \\
    1 & \cdots & \cdots & \cdots & \cdots & \cdots & 1
    \end{bmatrix}.
\end{equation*}

For the next three blocks $B_4,B_5,B_6$, we can apply Lemma~\ref{lemcmulti} with $c=(1,\cdots,1)$ on $B_4\in\mathcal{B}(2D,6,d_{embed})$ to 
copy the nonzero elements in the first row from columns $1,\cdots,2D$ to columns $2D+1,\cdots,4D$. Apply Lemma~\ref{lempp} on $B_5\in\mathcal{B}(2D,6,d_{embed})$ such that $h_{4,2D+i}$ interacts only with $h_{4,D+i}$, and $h_{4,3D+i}$ interacts only with $h_{4,D+i}$, $1\leq i\leq D$. Then apply Lemma~\ref{lemcmulti} with $c=(-1,\cdots,-1)$ on $B_6\in\mathcal{B}(2D,6,d_{embed})$ such that $h_{5,2D+i}$ interacts only with $h_{5,i}$ and $h_{5,3D+i}$ interacts only with $h_{5,D+i}$, $1\leq i\leq D$. 

Then we have
\begin{equation*}
    H_6:=B_6\circ B_5\circ\cdots\circ B_1(H)=
    \begin{bmatrix}
    x^1 & \cdots & x^D & \cdots & (x^1)^4 & \cdots & (x^D)^4 & \mathbf{0} \\
    0 & \cdots & \cdots & \cdots & \cdots & \cdots & \cdots & 0 \\
    \mathcal{I}_1 & \cdots & \cdots & \cdots & \cdots & \cdots  &  \cdots  & \mathcal{I}_{\ell} \\
    1 & \cdots & \cdots & \cdots  & \cdots & \cdots & \cdots & 1
    \end{bmatrix}.
\end{equation*}

Now suppose in the $(s-1)$-th step, we have 
\begin{equation*}
    H_{3s-3}:=B_{3s-3}\circ\cdots\circ B_{1}(H)=
    \begin{bmatrix}
    x^1 & \cdots & x^D & \cdots & (x^1)^{2^{s-1}} & \cdots & (x^D)^{2^{s-1}} & \mathbf{0} \\
    0 & \cdots & \cdots & \cdots & \cdots & \cdots & \cdots & 0 \\
    \mathcal{I}_1 & \cdots & \cdots & \cdots & \cdots & \cdots & \cdots & \mathcal{I}_{\ell} \\
    1 & \cdots & \cdots & \cdots & \cdots & \cdots & \cdots & 1
    \end{bmatrix}.
\end{equation*}
 
Then we can apply Lemma~\ref{lemcmulti} with $c=(1,\cdots,1)$ on $B_{3s-2}\in \mathcal{B}(2^{s-1}D,6,d_{embed})$ to copy the nonzero elements in the first row from columns $1,\cdots,2^{s-1}D$ to columns $2^{s-1}D+1,\cdots,2^sD$. Apply Lemma~\ref{lempp} on $B_{3s-1}\in\mathcal{B}(2^{s-1}D,6,d_{embed})$ to build $2^{s-1}D$ attention heads such that $h_{3s-2,(2^{s-1}+j-1)D+i}$ interacts only with $h_{3s-2,(2^{s-1}-1)D+i}$, for $1\leq j\leq 2^{s-1}$ and $1\leq i\leq D$. Apply Lemma~\ref{lemcmulti} with $c=(-1,\cdots,-1)$ on $B_{3s}\in\mathcal{B}(2^{s-1}D,6,d_{embed})$ to build $2^{s-1}D$ attention heads such that $h_{3s-1,2^{s-1}D+i}$ interacts only with $h_{3s-1,i}$, for $1\leq i\leq 2^{s-1}D$.

Therefore, we get 
\begin{align*}
    B_{3s}\circ B_{3s-1}\circ\cdots\circ B_{1}(H)
    &=B_{3s}\circ B_{3s-1}\circ B_{3s-2}(H_{3s-3}) \\
    &= \begin{bmatrix}
    x^1 & \cdots & x^D & \cdots & (x^1)^{2^s} & \cdots & (x^D)^{2^s} & \mathbf{0} \\
    0 & \cdots & \cdots & \cdots & \cdots & \cdots & \cdots & 0 \\
    \mathcal{I}_1 & \cdots & \cdots & \cdots & \cdots & \cdots & \cdots & \mathcal{I}_{\ell} \\
    1 & \cdots & \cdots & \cdots & \cdots & \cdots & \cdots & 1
    \end{bmatrix},
\end{align*}
as desired. The weights $\|\theta_B\|_\infty\leq O(\ell^2M^2\|H\|^2_{\infty,\infty})$ follows from Lemma~\ref{IAlemma}.

By reexamining the proof, the total number of attention heads needed in this implementation is $3\cdot2D(1+2+ \cdots + 2^{s-1})=6D(2^{s}-1)=6D(r-1)$.
\end{proof}

\subsubsection{Proof of Lemma~\ref{lemdiv}}

\begin{proof} [Proof of Lemma~\ref{lemdiv}]
For power series,
it suffices to show for the case $r=2^s$. First, by Lemma~\ref{lemrpower}, we can construct $B_i\in\mathcal{B}(2^{\lfloor i/2 \rfloor},6,d_{embed})$, $1\leq i\leq 3s$, such that 
    \begin{equation*}
    H_{3s}: =B_{3s}\circ\cdots\circ B_{1}(H)=
    \begin{bmatrix}
    (x^1)^1 & \cdots &(x^1)^r & \mathbf{0} \\
    0 & \cdots & \cdots & 0\\
    \mathcal{I}_1 & \cdots & \cdots  & \mathcal{I}_{\ell} \\
    1 & \cdots & \cdots & 1 \\ \end{bmatrix}.
\end{equation*}
Then by Lemma~\ref{lemsumD}, we can construct $B_{3s+1}\in\mathcal{B}(r,6,d_{embed})$ such that 
\begin{equation*}
    B_{3s+1}(H_{3s})=B_{3s+1}\circ\cdots\circ B_{1}(H)=
    \begin{bmatrix}
    (x^1)^1 & \cdots &(x^1)^r & \sum_{i=1}^r (x^1)^i & \mathbf{0} \\
    0 & \cdots & \cdots & \cdots & 0\\
    \mathcal{I}_1 & \cdots & \cdots & \cdots  & \mathcal{I}_{\ell} \\
    1 & \cdots &\cdots & \cdots & 1 \\ 
    \end{bmatrix}.
\end{equation*}
For division, it suffices to show for the case $r=2^s$ as well. First, by Lemma~\ref{lemcmulti} and Lemma~\ref{lemadd}, we can construct $B_1,B_2\in\mathcal{B}(1,6,d_{embed})$ such that 
\begin{equation*}
    B_2\circ B_{1}(H)=
    \begin{bmatrix}
    x^1 & -cx^1 & 1-cx^1 & \mathbf{0}  \\
    0 & \cdots & \cdots &  0\\
    \mathcal{I}_1 & \cdots & \cdots  & \mathcal{I}_{\ell} \\
    1 & \cdots & \cdots & 1 \\ 
    \end{bmatrix}.
\end{equation*}
Then by the first part of this proof, we can construct $B_i\in\mathcal{B}(2^{\lfloor (i-3)/3 \rfloor},6,d_{embed})$, $3\leq i\leq 3s+2$, to implement all the $i$-th power of $(1-cx^1)^i$, $1\leq i\leq r$. Then we can construct $B_{3s+3}\in\mathcal{B}(r,6,d_{embed})$ to add up all the powers, i.e.,
\begin{equation*}
    B_{3s+3}\circ\cdots \circ B_1(H)=
    \begin{bmatrix}
    x^1 & -cx^1 & 1-cx^1 & (1-cx^1)^2 &\cdots & \sum_{i=1}^r(1-cx^1)^i & \mathbf{0}  \\
    0 & \cdots & \cdots & \cdots & \cdots & \cdots & 0\\
    \mathcal{I}_1 & \cdots & \cdots  & \cdots & \cdots &  \cdots & \mathcal{I}_{\ell} \\
    1 & \cdots & \cdots & \cdots & \cdots & \cdots & 1 \\ 
    \end{bmatrix}.
\end{equation*}
Then, we apply Lemma~\ref{lemadd} and Lemma~\ref{lemcmulti} to construct $B_{3s+4},B_{3s+5}\in\mathcal{B}(1,6,d_{embed})$ to add the constant $1$ into the power series and multiply the constant $c$ respectively, i.e., 
\begin{equation*}
    B_{3s+5}\circ\cdots\circ B_1(H)=
    \begin{bmatrix}
    x^1 & -cx^1 & 1-cx^1 &  \cdots & \sum_{i=0}^r(1-cx^1)^i & c\sum_{i=0}^r(1-cx^1)^i & \mathbf{0}  \\
    0 & \cdots & \cdots & \cdots & \cdots & \cdots & 0\\
    \mathcal{I}_1 & \cdots & \cdots  & \cdots & \cdots &  \cdots & \mathcal{I}_{\ell} \\
    1 & \cdots & \cdots & \cdots & \cdots & \cdots & 1 \\ 
    \end{bmatrix}.
\end{equation*}
Since
\begin{equation*}
    \left|\frac{1}{x^1}-c\sum_{i=0}^r(1-cx^1)^i\right|=\left| c\sum_{i=r+1}^{\infty}(1-cx^1)^i\right|=\left|\frac{(1-cx^1)^{r+1}}{x^1}\right|,
\end{equation*}
we get the desired approximation result. The weights $\|\theta_B\|_\infty\leq O(\ell^2M^2\|H\|^2_{\infty,\infty})$ follows from Lemma \ref{IAlemma}.
\end{proof}

\begin{remark} \label{rmkdivision}
    For any $x\in [c_1, c_2]$ with $0<c_1<c_2$, i.e., $x$ is bounded above and bounded away from 0, we can find some $c$ such that $1-cx\in (-1,1)$. Given any prescribed tolerance $\epsilon>0$, by solving $(1-cx)^{r+1}/x\leq\epsilon$, we get $r=O(\ln(\frac{1}{\epsilon}))$. This is useful when calculating the depth $L_T$ and token number $m_T$ of each block in the transformer network when approximating each $\eta_i(x)$ in Proposition~\ref{approxetai}.
\end{remark}

\section{Proof of Proposition~\ref{repetatildei} and \ref{approxetai}} \label{secAppendixprop12}

\begin{proof} [Proof of Proposition~\ref{repetatildei}]
Notice that the two key components in $\tilde\eta_i(x)$: 
\[-\left(\frac{\|P(z_i)^{\top}(x-z_i)\|_2}{h\delta}\right)^2 \quad \text{and} \quad -\left(\frac{\|x-z_i\|_2}{p\tau_{\mathcal{M}}(z_i)}\right)^2\] have no interaction between each other, therefore can be built in parallel using the same number of transformer blocks. Let us focus on implementing $-\left(\frac{\|P(z_i)^{\top}(x-z_i)\|_2}{h\delta}\right)^2$.

Let $x\in\mathbb{R}^D$, for each $i=1,\cdots,K$, we first embed $x$ into the embedding matrix $H$ where
\begin{equation*}
H=
    \begin{bmatrix}
    x^1 & \cdots & x^D & \mathbf{0}  \\
    0 & \cdots & \cdots & 0 \\
    \mathcal{I}_1 & \cdots & \cdots & \mathcal{I}_{\ell} \\
    1 & \cdots & \cdots & 1
\end{bmatrix}\in\mathbb{R}^{d_{embed}\times\ell}.
\end{equation*}

\noindent
$\bullet$ \textbf{Implementation of } $x - z_i$ 
\\
\noindent
By Lemma~\ref{lemadd}, we can construct $B_1\in\mathcal{B}(D,6,d_{embed})$ so that it implements the constant addition $x-z_i$ in the first row from columns $D+1$ to $2D$, i.e.,
\begin{equation*}
    H_1:=B_1(H)=
    \begin{bmatrix}
    x^1 & \cdots & x^D & x^1-(z_i)^1 & \cdots & x^D-(z_i)^D & \mathbf{0} \\
    0 & \cdots & \cdots & \cdots & \cdots & \cdots & 0 \\
    \mathcal{I}_1 & \cdots & \cdots & \cdots & \cdots & \cdots & \mathcal{I}_{\ell} \\
    1 & \cdots & \cdots & \cdots & \cdots & \cdots & 1
    \end{bmatrix}.
\end{equation*}
\noindent
$\bullet$ \textbf{Implementation of} $P(z_i)^\top(x-z_i)$ 
\\
\noindent
By Lemma~\ref{lemcmulti}, we can sequentially construct $B_2,B_3,\cdots,B_{d+1}\in\mathcal{B}(D,6,d_{embed})$ so that each of them implements the constant multiplication with $c_j=(P(z_i)^\top_{j,1},\cdots,P(z_i)^\top_{j,D}) = (P(z_i)_{1,j},\cdots,P(z_i)_{D,j})$ for $j=1,\cdots,d$. 
For each $j=1,\cdots,d$, we put the constant multiplication results 
\[\left(P(z_i)_{1,j}(x^1-(z_i)^1),\cdots,P(z_i)_{D,j}(x^D-(z_i)^D)\right)\] in the first row from columns $(j+1)D+1$ to $(j+2)D$, i.e.,

\[
H_{d+1}:=B_{d+1}\circ\cdots\circ B_1(H)=
\begin{bmatrix}
    \begin{array}{c|ccccc}
     \multirow{4}{*}{$(H_1)_{:,I_1}$} & 
    P(z_i)_{1,1}(x^1-(z_i)^1) & \cdots & \cdots & P(z_i)_{D,d}(x^D-(z_i)^D) & \mathbf{0} \\
    & 0 & \cdots &  \cdots & \cdots & 0 \\
    & \mathcal{I}_{2D+1}  & \cdots & \cdots & \cdots & \mathcal{I}_{\ell} \\
    & 1 &  \cdots & \cdots & \cdots & 1
    \end{array}
\end{bmatrix},
    \]
where $I_1=\{1,\cdots,2D\}$. The notation $(H_1)_{:,I_1}$ denotes the submatrix of $H_1$ with all the rows and columns with column index in $I_1$.

Next, by Lemma~\ref{lemsumD}, we can construct $B_{d+2}\in\mathcal{B}(D,6,d_{embed})$ so that it implements the sum of the terms in the first row of $H_{d+1}$ block by block, where each block is a sum of $D$ terms, and we put the $d$ sums in the first row from columns $(d+2)D+1$ to $(d+2)D+d$. More precisely, we have 
\begin{align*}
H_{d+2}:&=B_{d+2}(H_{d+1}) \\
&=
\begin{bmatrix}
    \begin{array}{c|cccc}
     \multirow{4}{*}{$(H_{d+1})_{:,I_{d+1}}$} & 
    \sum_{j=1}^DP(z_i)_{j,1}(x^j-(z_i)^j) & \cdots & \sum_{j=1}^DP(z_i)_{j,d}(x^j-(z_i)^j) & \mathbf{0} \\
    & 0 & \cdots & \cdots & 0 \\
    & \mathcal{I}_{(d+2)D+1}  & \cdots & \cdots & \mathcal{I}_{\ell} \\
    & 1 &  \cdots & \cdots & 1
    \end{array}
\end{bmatrix},
\end{align*}
where $I_{d+1}=\{1,\cdots,(d+2)D\}$.
\vspace{5mm}
\\
\noindent
$\bullet$ \textbf{Implementation of} $-\left(\frac{\|P(z_i)^{\top}(x-z_i)\|_2}{h\delta}\right)^2$ 
\\
\noindent
Then by Lemma~\ref{lemsq}, we can construct $B_{d+3}\in\mathcal{B}(D,6,d_{embed})$ so that it implements the square of those sums in the first row of $H_{d+2}$, and we put the corresponding squares in the first row from columns $(d+2)D+d+1$ to $(d+2)D+2d$. Thus, 
\begin{align*}
H_{d+3}:&=B_{d+3}(H_{d+2}) \\
&=
\begin{bmatrix}
    \begin{array}{c|cccc}
     \multirow{4}{*}{$(H_{d+2})_{:,I_{d+2}}$} & 
    \left(\sum_{j=1}^DP(z_i)_{j,1}(x^j-(z_i)^j)\right)^2 & \cdots & \left(\sum_{j=1}^DP(z_i)_{j,d}(x^j-(z_i)^j)\right)^2 & \mathbf{0} \\
    & 0 & \cdots & \cdots & 0 \\
    & \mathcal{I}_{(d+2)D+d+1}  & \cdots & \cdots & \mathcal{I}_{\ell} \\
    & 1 &  \cdots & \cdots & 1
    \end{array}
\end{bmatrix},
\end{align*}
where $I_{d+2}=\{1,\cdots,(d+2)D+d\}$.

Finally, by Lemma~\ref{lemsumD}, we can construct $B_{d+4}\in\mathcal{B}(D,6,d_{embed})$ and $B_{d+5}\in\mathcal{B}(1,6,d_{embed})$ so that $B_{d+4}$ implements the sum of those squares in $H_{d+3}$, i.e., it computes the square of $2$-norm of the term $\|P(z_i)^{\top}(x-z_i)\|_2^2$, and $B_{d+5}$ implements the constant $-1/(h\delta)^2$ multiplication. Therefore, 

\[
H_{d+5}:=B_{d+5}\circ B_{d+4}(H_{d+3})=
\begin{bmatrix}
    \begin{array}{c|cccc}
     \multirow{4}{*}{$(H_{d+3})_{:,I_{d+3}}$} & 
     \|P(z_i)^{\top}(x-z_i)\|^2_2 & -\left(\frac{\|P(z_i)^{\top}(x-z_i)\|_2}{h\delta}\right)^2 & \mathbf{0} \\
    & 0 & \cdots &  0 \\
    & \mathcal{I}_{(d+2)D+2d+1}  & \cdots & \mathcal{I}_{\ell} \\
    & 1 &  \cdots &  1
    \end{array}
\end{bmatrix},
    \]
where $I_{d+3}=\{1,\cdots,(d+2)D+2d\}$. The total number hidden tokens is on the order of $O(Dd)$.
\vspace{5mm}



\noindent
$\bullet$ \textbf{Implementation of } $-\left(\frac{\|x-z_i\|_2}{p\tau_{\mathcal{M}}(z_i)}\right)^2$ 
\\
\noindent
For the implementation of $-\left(\frac{\|x-z_i\|_2}{p\tau_{\mathcal{M}}(z_i)}\right)^2$, we need $D$ more tokens to save the values 
\[(x^1-(z_i)^1)^2,\cdots,(x^D-(z_i)^D)^2,\]
$1$ more token to save the $2$-norm square $\|x-z_i\|_2^2=\sum_{j=1}^D (x^j-(z_i)^j)^2$, and $1$ more token to save the constant multiplication with constant $-1/(p\tau_{\mathcal{M}}(z_i))^2$. By the Interaction Lemma~\ref{IAlemma}, we can implement all these operation in parallel within transformer blocks $B_{d+3},B_{d+4},B_{d+5}$ for the implementation of $-\left(\frac{\|P(z_i)^{\top}(x-z_i)\|_2}{h\delta}\right)^2$. We need $D+2$ more tokens for this. So far, after bringing the implementation of $-\left(\frac{\|x-z_i\|_2}{p\tau_{\mathcal{M}}(z_i)}\right)^2$, we have 

\[
H_{d+5}=
\begin{bmatrix}
    \begin{array}{c|cccc}
     \multirow{4}{*}{$(H_{d+4})_{:,I_{d+4}}$} & 
     -\left(\frac{\|x-z_i\|_2}{p\tau_{\mathcal{M}}(z_i)}\right)^2 & -\left(\frac{\|P(z_i)^{\top}(x-z_i)\|_2}{h\delta}\right)^2 & \mathbf{0} \\
    & 0 & \cdots &  0 \\
    & \mathcal{I}_{(d+3)D+2d+3}  & \cdots & \mathcal{I}_{\ell} \\
    & 1 &  \cdots &  1
    \end{array}
\end{bmatrix},
    \]
where $I_{d+4}=\{1,\cdots,(d+3)D+2d+2\}$.
\vspace{5mm}
\\
\noindent
$\bullet$ \textbf{Implementation of } $1-\left(\frac{\|x-z_i\|_2}{p\tau_{\mathcal{M}}(z_i)}\right)^2-\left(\frac{\|P(z_i)^{\top}(x-z_i)\|_2}{h\delta}\right)^2$ 
\\
\noindent
Furthermore, we need $B_{d+6}\in\mathcal{B}(2,6,d_{embed})$ to take the sum of $-\left(\frac{\|x-z_i\|_2}{p\tau_{\mathcal{M}}(z_i)}\right)^2$ and $-\left(\frac{\|P(z_i)^{\top}(x-z_i)\|_2}{h\delta}\right)^2$, and $B_{d+7}\in\mathcal{B}(1,6,d_{embed})$ to add constant $1$, i.e.,
\begin{align*}
H_{d+7}:&=B_{d+7}\circ B_{d+6}(H_{d+5}) \\ 
&=
\begin{bmatrix}
    \begin{array}{c|cccc}
     \multirow{4}{*}{$(H_{d+5})_{:,I_{d+5}}$} & 
     -\left(\frac{\|x-z_i\|_2}{p\tau_{\mathcal{M}}(z_i)}\right)^2-\left(\frac{\|P(z_i)^{\top}(x-z_i)\|_2}{h\delta}\right)^2  & 1-\left(\frac{\|x-z_i\|_2}{p\tau_{\mathcal{M}}(z_i)}\right)^2-\left(\frac{\|P(z_i)^{\top}(x-z_i)\|_2}{h\delta}\right)^2  & \mathbf{0} \\
    & 0 & \cdots &  0 \\
    & \mathcal{I}_{(d+3)D+2d+5}  & \cdots & \mathcal{I}_{\ell} \\
    & 1 &  \cdots &  1
    \end{array}
\end{bmatrix},
\end{align*}
where $I_{d+5}=\{1,\cdots,(d+3)D+2d+4\}$. 
\vspace{5mm}
\\
\noindent
$\bullet$ \textbf{Implementation of } $\tilde\eta_i(x)$ 
\\
\noindent
Finally, we need one block $B_{d+8}$ to implement the ReLU function. This can be achieved by the similar spirit as the proof of Lemma~\ref{lemcmulti}.

For $B_{d+8}$, let us define an attention head $A$ with the data kernel in the form
\begin{equation*}
Q_i^{data}
=
\begin{bmatrix}
    0 & 0 & 0 & 0  & 1 \\
    0 & 0 & 0 & 0  & 1
\end{bmatrix}
\quad 
K_i^{data}
=
\begin{bmatrix}
    1 & 0 & 0 & 0  & 0 \\
    0 & 0 & 0 & 0  & 0
\end{bmatrix}.
\end{equation*}
By Interaction Lemma \ref{IAlemma}, we can construct $A$ in such a way that $h_{d+7,(d+3)D+2d+7}$ interacts with $h_{d+7,(d+3)D+2d+6}$ only, i.e.,
\begin{align*}
    A(h_{d+7,(d+3)D+2d+7}) &= \sigma(\langle Q_i^{data}h_{d+7,(d+3)D+2d+7}, K_i^{data}h_{d+7,(d+3)D+2d+6} \rangle)e_1 \\
    &=
    \sigma\left(1-\left(\frac{\|x-z_i\|_2}{p\tau_{\mathcal{M}}(z_i)}\right)^2-\left(\frac{\|P(z_i)^{\top}(x-z_i)\|_2}{h\delta}\right)^2\right)e_1 \\ &=\tilde\eta_i(x)e_1,
\end{align*}
and $A_i(h_{d+7,t})=0$ when $t\neq (d+3)D+2d+7$. For the feed-forward layer of $B_8$, we take the weight matrix equals to identity and bias equals to zero, so that it implements the identity operation. It is easy to see $B_{d+8}\in\mathcal{B}(1,1,d_{embed})$ and 
\begin{align*}
H_{d+8}:=B_{d+8}(H_{d+7})=
\begin{bmatrix}
    \begin{array}{c|cccc}
     \multirow{4}{*}{$(H_{d+7})_{:,I_{d+7}}$} & \tilde\eta_i(x)  & \mathbf{0} \\
    & 0 &  0 \\
    & \mathcal{I}_{(d+3)D+2d+7}   & \mathcal{I}_{\ell} \\
    & 1  &  1
    \end{array}
\end{bmatrix}, 
\end{align*}
where $I_{d+7}=\{1,\cdots,(d+3)D+2d+6\}$.

By reexamining the proof, we get $L_T=O(d)$, $m_T=O(D)$, $d_{embed}=5$, $\ell\geq O(Dd)$, $L_{FFN}=6$, $w_{FFN}=5$, $\kappa=O(D^2d^6\delta^{-8})$. By hiding the dependency on $d$ when it is not the dominating term, we have $L_T=O(d)$, $m_T=O(D)$, $d_{embed}=5$, $\ell\geq O(D)$, $L_{FFN}=6$, $w_{FFN}=5$, $\kappa=O(D^2\delta^{-8})$.
\end{proof}

\begin{remark} \label{rmkparalell}
    The above procedure implements of one $\tilde\eta_i(x)$, for $i=1,\cdots,K$. To implement all $\tilde\eta_1(x),\cdots,\tilde\eta_K(x)$ parallely, we can start with a large $\ell$ and partition the matrix into $K$ chunks where each chunk implements one of $\tilde\eta_i(x)$. Such implementation is possible because of the Interaction Lemma~\ref{IAlemma}. Moreover, as discussed in Remark~\ref{rmkflexibility}, each intermediate output can be put into any column in the matrix without affecting the final result. This flexibility also facilitates parallelization.
\end{remark}



\begin{proof} [Proof of Proposition~\ref{approxetai}]
First, we would like to parallelize (see also Remark~\ref{rmkparalell}) apply Proposition~\ref{repetatildei} to implement $\tilde\eta_1(x),\cdots,\tilde\eta_K(x)$ simultaneously. 
Let $H$ be an embedding matrix of the form
\[H=
    \begin{bmatrix}
    \begin{array}{ccc|ccccc}
   \multirow{4}{*}{$(H_{d+7})_{:,I^1_{d+7}}$} & \multirow{4}{*}{$\cdots$} & \multirow{4}{*}{$(H_{d+7})_{:,I^K_{d+7}}$} & \tilde\eta_1(x) & \cdots & \tilde\eta_K(x) &\|\tilde\eta(x)\|_1 & \mathbf{0}  \\
    & & & 0 & \cdots & \cdots & \cdots  & 0 \\
    & & & \mathcal{I}_{((d+3)D+2d+6)K+1} & \cdots &  \cdots & \cdots & \mathcal{I}_{\ell} \\
    & & & 1 & \cdots &  \cdots & \cdots & 1
    \end{array}
\end{bmatrix}.
\]
From Theorem 2.2 in \citep{Cloninger21}
, we know $K=O(\delta^{-d})$ where $O(\cdot)$ hides the dependency on $d$ and ${\rm Vol}(\mathcal{M})$. Thus, there exists $T_{1}(\theta;\cdot)\in\mathcal{T}$ with $L_T=O(d)$, $m_T=O(KD)=O(D\delta^{-d})$, $d_{embed}=5$, $\ell\geq O(KD)=O(D\delta^{-d})$, $L_{FFN}=6$, $w_{FFN}=5$, $\kappa=O(D^2\delta^{-2d})$ such that $T_1(\theta;\cdot)$ can exactly represent $H$.  

Then, by Lemma~\ref{lemdiv}, we can construct transformer blocks $B_1,\cdots,B_{3s+5}$ with the maximum number of attention heads equal to $r$ within each block to approximate $\frac{1}{\|\tilde\eta(x)\|_1}$ up to $\left|\frac{(1-c\|\tilde\eta(x)\|_1)^{r+1}}{\|\tilde\eta(x)\|_1}\right|$ tolerance, where $c$ is some constant such that $1-c\|\tilde\eta(x)\|_1\in (-1,1)$. 
As shown in Proposition 6.3 of \citep{Cloninger21}, that $1-q\lesssim \|\tilde\eta(x)\|_1 \lesssim d^{d/2}(1-q)^{-2d}$, where $\lesssim$ hides the dependency of some absolute constants. Therefore we can find some $c$ such that $1-c\|\tilde\eta(x)\|_1\in (-1,1)$. More precisely,
\begin{align*}
    H_{3s+5}:&=B_{3s+5}\circ\cdots\circ B_1(H) \\
    &=
    \begin{bmatrix}
    \begin{array}{c|ccccccc}
     \multirow{4}{*}{$\cdots$} &  \tilde\eta_1(x) & \cdots & \tilde\eta_K(x) &
    \|\tilde\eta(x)\|_1  &  \cdots & c\sum_{k=0}^r(1-c\|\tilde\eta(x)\|_1)^k & \mathbf{0}  \\
    & 0 & \cdots & \cdots & \cdots & \cdots & \cdots & 0 \\
    & \mathcal{I}_{((d+3)D+2d+6)K+1} & \cdots & \cdots & \cdots & \cdots  &  \cdots & \mathcal{I}_{\ell} \\
    & 1 & \cdots & \cdots & \cdots & \cdots & \cdots & 1 \\ 
    \end{array}
    \end{bmatrix}.
\end{align*}
Then, by Lemma~\ref{lempp}, for each fixed $i=1,\cdots,K$, we can construct $B^i_{3s+6}\in\mathcal{B}(1,6,d_{embed})$ such that it implements the pairwise multiplication between $c\sum_{k=0}^r(1-c\|\tilde\eta(x)\|_1)^k$ and $\tilde\eta_i(x)$, i.e.,
\begin{align*}
    H^i_{3s+6}:=B^i_{3s+6}(H_{3s+5})
    =
    \begin{bmatrix}
    \begin{array}{c|ccccc}
    \multirow{4}{*}{$\cdots$} &
    c\sum_{k=0}^r(1-c\|\tilde\eta(x)\|_1)^k  &  c\tilde\eta_i(x)\sum_{k=0}^r(1-c\|\tilde\eta(x)\|_1)^k & \mathbf{0}  \\
    & 0 & \cdots & 0 \\
    & \mathcal{I}_{((d+3)D+2d+6)K+K+r+4} &  \cdots & \mathcal{I}_{\ell} \\
    & 1 & \cdots & 1 \\ 
    \end{array}
    \end{bmatrix}.
\end{align*}

Since $\frac{1}{t}=c\sum_{k=0}^\infty(1-ct)^k$ for $1-ct\in (-1,1)$, we can truncate the approximation of 
$\frac{1}{\|\tilde\eta(x)\|_1}$ up to $r$-th power such that 

\begin{equation} \label{trunc}
\begin{split}
    \left|\frac{1}{\|\tilde\eta(x)\|_1}-c\sum_{k=0}^r \left(1-c\|\tilde\eta(x)\|_1\right)^k\right|&=\left|c\sum_{k=r+1}^\infty \left(1-c\|\tilde\eta(x)\|_1\right)^k\right| \\
    &=\left|\frac{(1-c\|\tilde\eta(x)\|_1)^{r+1}}{\|\tilde\eta(x)\|_1}\right|\leq\frac{\epsilon}{\|\tilde\eta(x)\|_1}.
\end{split}
\end{equation}
Therefore 
\begin{equation*}
    \left|\eta_i(x)-c\tilde\eta_i(x)\sum_{k=0}^r \left(1-c\|\tilde\eta(x)\|_1\right)^k\right|=\left|\frac{\tilde\eta_i(x)}{\|\tilde\eta(x)\|_1}-c\tilde\eta_i(x)\sum_{k=0}^r \left(1-c\|\tilde\eta(x)\|_1\right)^k\right|\leq\frac{\epsilon\tilde\eta_i(x)}{\|\tilde\eta(x)\|_1}=\epsilon\eta_i(x).
\end{equation*}
From the last inequality of (\ref{trunc}), we get $r=O(\ln(\frac{1}{\epsilon}))$ (See also Remark~\ref{rmkdivision}).
Let $T^i_2$ implements the sequence $B_{3s+6}\circ\cdots\circ B_1$ for each fixed $i$, then each $T_2^i$ satisfies $L_{T_2^i}=O(\ln(r))=O(\ln(\ln(\frac{1}{\epsilon})))$ and $m_{T_2^i}=r=O(\ln(\frac{1}{\epsilon}))$. 

Let $T_2:=(T_2^1,\cdots,T_2^K)$, then $T_2$ satisfies $L_{T_2}=O(\ln(r))=O(\ln(\ln(\frac{1}{\epsilon})))$ and $m_{T_2}=O(\ln(\frac{1}{\epsilon})+K)$. Let $T:=T_2\circ T_1$, then we have
\[\sup_{x\in\mathcal{M}(q)}\|T(\theta;x)-\eta(x)\|_1=\sup_{x\in\mathcal{M}(q)}\sum_{i=1}^K\left|\eta_i(x)-c\tilde\eta_i(x)\sum_{k=0}^r \left(1-c\|\tilde\eta(x)\|_1\right)^k\right|\leq\sup_{x\in\mathcal{M}(q)}\sum_{i=1}^K\frac{\epsilon\tilde\eta_i(x)}{\|\tilde\eta(x)\|_1}=\epsilon,\] 
as desired.

By reexamining the proof, we get $L_T=L_{T_1}+L_{T_2}=O(d+\ln(\ln(\frac{1}{\epsilon})))$, $m_T=\max(m_{T_1},m_{T_2}) = O(\max(D\delta^{-d},\ln(\frac{1}{\epsilon})+K))=O(D\delta^{-d})$, $d_{embed}=5$, $\ell\geq O(D\delta^{-d}+\ln(\frac{1}{\epsilon})+K)=O(D\delta^{-d})$, $L_{FFN}=6$, $w_{FFN}=5$, $\kappa=O(D^2\delta^{-2d-8})$. 
\end{proof}

\begin{remark} \label{log}
When calculating the transformer network parameters, we make the assumption that the logarithmic term $\ln(\frac{1}{\epsilon})$ is much smaller than the exponential term $\delta^{-d}$. Although it is not always the case, we later on set $\epsilon=\delta^\alpha$ for some H\"{o}lder exponent $\alpha\in (0,1]$. This makes it a reasonable assumption.
    
\end{remark}

\section{Other Useful Lemmas}
\begin{lemma}[\citet{Havrilla24}]\label{coveringlemma}
    Let $\delta>0$, consider a transformer network class \\
    $\mathcal{T}(L_T,m_T,d_{embed},\ell,L_{\rm FFN},w_{\rm FFN},R,\kappa)$ with input $x\in\mathbb{R}^D$ satisfying $\|x\|_{\infty}\leq M$. Then
    \begin{equation}
        \mathcal{N}\left(\delta,\mathcal{T},\|\cdot\|_{\infty}\right)\leq \left(\frac{2^{L_T+1}L_{\rm FFN}M^{3L_T}d_{embed}^{18L_T^2}w_{\rm FFN}^{18L_T^2L_{\rm FFN}}\kappa^{6L_T^2L_{\rm FFN}}m_{\rm T}^{L_T^2}\ell^{L_T^2}}{\delta}\right)^{4d^2_{embed}w^2_{\rm FFN}D(m_{\rm T}+L_{\rm FFN})L_T}.
    \end{equation}
\end{lemma}
\begin{proof}
    We refer its proof to Lemma 2 in \citep{Havrilla24}.
\end{proof}


\end{document}